%% file: resource_alloc.tex
\documentclass[10pt]{article} 
\usepackage[accepted]{tmlr}

\input{math_commands.tex}

\usepackage{mathtools, bm}
\usepackage{amssymb, bm}
\usepackage{amsmath, amsthm, amssymb}

\usepackage{hyperref}
\usepackage{comment}
\usepackage{breqn}

\usepackage{graphicx}
\usepackage{subfigure}

\usepackage{algorithm2e}
\PassOptionsToPackage{algo2e,ruled}{algorithm2e}
\RequirePackage{algorithm2e}

\newtheorem{theorem}{Theorem}
\newtheorem{lemma}{Lemma}

\newtheorem{assumption}{Assumption}

\SetKw{Break}{Break}

\newcommand{\D}{\mathcal{D}}

\newcommand{\F}{\mathcal{F}}

\usepackage{mathrsfs}

\newcommand{\pr}[1]{\mathbb{P}\left(#1\right)}

\makeatletter
\newtheorem{repeatthm@}{Theorem}
\newenvironment{repeatthm}[1]{%
    \def\therepeatthm@{\ref{#1}}
    \repeatthm@
}
{\endrepeatthm@}
\makeatother

\makeatletter
\newtheorem{repeatlem@}{Lemma}
\newenvironment{repeatlem}[1]{%
    \def\therepeatlem@{\ref{#1}}
    \repeatlem@
}
{\endrepeatlem@}
\makeatother

\author{\name Juliette Achddou \email juliette.achdou@gmail.com \\
      \addr Department of Computer Science\\
      Università degli Studi di Milano
      \AND
      \name Olivier Cappé \email olivier.cappe@cnrs.fr \\
      \addr Department of Computer Science\\
      ENS Paris
      \AND
      \name Aurélien Garivier \email aurelien.garivier@ens-lyon.fr\\
      \addr Department of Mathematics \\
      ENS Lyon}

\def\month{04}  
\def\year{2024} 
\def\openreview{\url{https://openreview.net/forum?id=m1OXBLH0dH}} 

\begin{document}

\title{Stochastic Direct Search Methods\\ for Blind Resource Allocation}
\maketitle

\begin{abstract}
Motivated by programmatic advertising optimization, we consider the task of
sequentially allocating budget across a set of resources. At every time step,
a feasible allocation is chosen and only a corresponding random return is observed.
The goal is to maximize the cumulative expected sum of returns. This is a realistic
model for budget allocation across subdivisions of marketing campaigns, with the
objective of maximizing the number of conversions. 
We study direct search (also known as pattern search) methods for linearly constrained
and derivative-free optimization in the presence of noise, which apply in particular 
to sequential budget allocation. These algorithms, which do
not rely on hierarchical partitioning of the resource space, are easy to implement; they
respect the operational constraints of resource allocation by avoiding
evaluation outside of the feasible domain; and they are also compatible with warm start by
being (approximate) descent algorithms. However, they have not yet been analyzed from the
perspective of cumulative regret. We show that direct search methods achieves finite regret
in the deterministic and unconstrained case. In the presence of evaluation noise and linear
constraints, we propose a simple extension of direct search that achieves a regret upper-bound
of the order of $T^{2/3}$. We also propose an accelerated version of the algorithm, relying on
repeated sequential testing, that significantly improves the practical behavior of the
approach.

\end{abstract}

\section{Introduction}
\subsection{Motivation: Blind Resource Allocation}
In the field of programmatic marketing, advertisers are
given daily budgets that they are required to entirely distribute
across a number of predefined subdivisions of a campaign. Their goal
is to maximize some notion of cumulative reward during the lifetime of the
campaign, corresponding to the number of clicks or purchases generated
by the campaign. The expected reward generated by each
subdivision every day is an unknown function of the daily budget
allocated to that campaign.  We focus on the context in which, when
choosing a specific allocation, the advertiser only observes a noisy
version of the total reward.  The optimization task faced by the
advertiser can thus be formalized as a continuous resource allocation
problem, under zeroth order and noisy feedback. Indeed, a key
operational constraint 
is the impossibility to directly
access higher-order (derivative) information. Furthermore, every noisy
evaluation of the objective function has a cost related to the value
of the function, that needs to be accounted for in the performance
criterion. 
The cumulative reward is thus more
relevant than alternative traditional measures of performance based,
for instance, on the distance to the optimum or the norm of the
gradient of the objective function reached after some iterations. Note that the resource allocation task that we consider is different from that in which the resource constraints are cumulative, i.e. where the budget spans the whole period instead of one day or time-step. Cumulative and step-level constraints lead to distinct optimization problems, none of which being a reduction of the other.
  
The blind resource allocation task may be seen as a specific instance of the more general model
of zeroth-order linearly constrained optimization considered in this paper. 
For resource allocation, we assume that the learner has access to $d+1\in \mathbb{N}$
different resources. To keep up with the dominant convention in
optimization, we consider a minimization problem for which the
costs may be thought of as minus the rewards. At each round
$t \in \{1,\ldots T\}$, the learner is allowed to choose her level of
consumption of each resource, on a continuous scale
from $0$ to $1$.  We impose that the consumption levels of all the
resources sum to $1$ (corresponding to the constraint of spending all the daily budget in the advertising context). The use of resource $i \in \{1,\ldots, d+1\}$ to
a level of $x_t^{(i)}$ generates an expected marginal cost
$w_i(x_t^{(i)})$.  Overall, the expected one-step cost of the learner
is given by $\sum_{i =1}^{d+1} w_i(x_t^{(i)})$, where the set of all
possible consumption levels $(x_t^{(1)}, \ldots, x_t^{(d+1)})$
corresponds to the $d$ dimensional simplex.
The goal of the learner is to sequentially minimize the expected
cumulative cost over $T$ evaluations of the function,
having access only to a noisy version of the expected cost associated
to the allocation tried at step $t$. Not only are the cost functions
$w_1$ to $w_{d+1}$ unknown, but one cannot observe their individual
outputs.\looseness-1
%
%

\subsection{Model}\label{sec:model}

While the main application of interest to us is resource allocation, our results are valid for more general
linearly constrained optimization problems. 
We  consider a generic optimization domain $\mathcal{D}$ that is a subset of
$\mathbb{R}^d $ defined by linear constraints:
$\mathcal{D} = \{x \in \mathbb{R}^d, A_I x \leq u \}$ with
$m \in \mathbb{N}$, $A_I \in \R^{m \times d},$ and $u \in \R^{m} $.
At each round $t \in \{1,\ldots T\}$, the learner selects
$x_t \in \mathcal{D}$ and incurs a cost $f(x_t) + \epsilon_t$, where
$\epsilon_t$ is assumed to be a centered $\sigma$-subgaussian noise,
with $\sigma$ known to the
learner. \looseness-1 

The goal of the learner is to minimize the cumulative cost over $T$
evaluations of the function, or equivalently to \emph{minimize the
  cumulative regret}
$$R_T = \sum_{t=1}^T f(x_t) - f(x_\star)\;, $$
denoting  by $x_\star = \argmin_{x \in \mathcal{D}} f(x)$ the optimal allocation. We make the following regularity assumption on $f$.\looseness-1

\begin{assumption}\label{ass:reg}
$f$ is continuously differentiable,  $\beta$-smooth and $a$-strongly convex on $\mathcal{D}$.
\end{assumption}
Note  that under Assumption \ref{ass:reg}, $f$
has a unique minimum and is bounded from below. Assumption
\ref{ass:reg} is  common in online
optimization due to the difficulty of controlling the cumulative cost
without this assumption.\looseness-1

In programmatic advertising and economics, it is common to observe marginal 
returns that decrease as the level of a resource increases 
(following the so-called law of diminishing returns). Assuming convexity 
of $f$ on $\mathcal{D}$ is therefore reasonable, since 
Assumption \ref{ass:reg} implies that when $f$ has the form
$ f(x_t^{(1)}, \ldots x_t^{(d)}) = \sum_{i=1}^{d} w_i(x_t^{(i)}) +
w_{d+1} \left(1 - \sum_{i=1}^{d}x_t^{(i)}\right)$, the marginal cost
functions $w_1$ to $w_{d+1}$ are also convex and hence satisfy the law
of diminishing return (viewing $-w_i$ as the marginal utility
associated to the $i$-th resource).\looseness-1

\subsection{Related Works}\label{sec:rel_works}
The discrete counterpart of the resource allocation model in which the
resources can only be used up to discrete consumption levels, is a
celebrated model of operations research with multiple
applications. Its properties have first been discussed by
\cite{koopman} who proposed the first algorithmic solution for this
problem. Koopman's works have further been extended by
\cite{gross,katoh1979polynomial} who propose more efficient algorithms
under specific assumptions on the number of resources and the total
consumption budget. The range of applications is wide, including
experimental design, load management in an industrial context,
computer scheduling and, more recently, the \emph{adwords} problem
introduced by
\cite{mehta2007adwords}. 
Recently, \cite{Agrawal:2015} studied online and offline resource
allocation, motivated by the latter
task. 
With the same motivation, \cite{fontaine2020adaptive} focus on the
online and continuous version of resource allocation in which the
learner accesses the derivatives. The method studied by
\cite{fontaine2020adaptive} extends the bisection method in dimension
$d>1$.


The broader problem of derivative-free optimization in noisy
environments has been considered by researchers coming from different
horizons. A relevant stream of works originates from the bandit
community, which considered this task as an extension of the more
traditional multi-armed bandit problem \citep[see
e.g][]{auer2002finite}. The class of $\mathcal{X}$-armed bandits
models focuses on the case where a learner can select actions in a
generic measurable space and the mean-payoff function is regular. In
\citep{bubeck2011x}, for example, the mean payoff function is supposed
to be locally Lipschitz with respect to some dissimilarity
measure. \cite{bubeck2011x} and \cite{munos2014bandits} adopt the
approach of hierarchical optimization, in which the optimization
domain is iteratively partitioned, resulting in finer and finer
partitions, that are required to be balanced in some sense. The
learner maintains an upper confidence bound of the goal function that
is constant on each cell defined by the finest partition.  The
algorithm proposed by \cite{bubeck2011x} achieves a regret of the
order of $\sqrt{T}$ when the learner knows the exact order of the
smoothness at the optimum. However, partitioning the domain in a
hierarchical and balanced way is relatively easy when the domain is an
hypercube, but is a computational problem in itself when the domain
has a more complex form. We also mention that knowing the smoothness
is considered a challenge most of the time in black-box optimization,
so that several methods have been introduced that are adaptive to the
smoothness
(\cite{locatelli2018adaptivity,valko2013stochastic,shang2019general}).
We mention that concurrently to HOO based on hierarchical
partitioning, \cite{agarwal2011stochastic} has also proposed a
different strategy for $\mathcal{X}$-armed bandits, but this time
convex, with ellipsoid methods, that also result in $O(\sqrt{T})$
regret. \looseness-1


The extension of more traditional first-order optimization methods has
also been considered. When the function
evaluation is not perturbed by any noise, \cite{nesterov2017random}
consider random gradient descent based on finite differences to
estimate the gradient. \cite{flaxman2004online} consider a
version of stochastic gradient descent with a one-point estimate of
the gradient for the adversarial setting introduced by
\cite{zinkevich2003online} in which at each time step, a new goal
function is chosen by an adversary, making it impossible to rely on a
two-point estimate of the gradient. In this setting,
\cite{flaxman2004online} show an adversarial regret bound of the order
of $T^{5/6}$.  Later on, \cite{hazan2014bandit},
\cite{hazan2016optimal}, \cite{bubeck2017kernel} propose new methods
for the same setting, but with adversarial convex or strongly-convex
functions, showing improved regret bounds, as low as $\sqrt{T}$.  In a
stochastic setting that is closer to ours,
\cite{akhavan2020exploiting,bach2016highly} consider a version of
stochastic gradient descent with unbiased estimates of the gradient,
obtained by finite differences. While they provide an analysis in term
of the regret, they focus on a restricted notion of regret that is
different from the one considered in this work. The algorithms that
they propose rely on a number of samples used to estimate the
gradient at each iteration of the gradient descent algorithm. But the
regret only accounts for the cost incurred by the iterates of the
gradient descent algorithm and ignores the regret incurred by the
samples used for the estimation of the gradient. Moreover, the
constraints also do not apply to those samples, meaning that the
algorithm is allowed to get samples outside of the feasible domain in
order to estimate the gradient. The authors prove an upper bound on their version of the
regret, which is of the order of $\sqrt{T}$ when Assumption
\ref{ass:reg} is satisfied. In Section~\ref{sec:illustration} below we
will see how evaluations outside of the feasible domain can be avoided,
using for instance ideas of \cite{bravo2018bandit}, at the price of
an increased regret rate. \looseness-1


\subsection{Contribution}

In this paper, we focus instead on a class of simple but mathematically
well grounded algorithms known as direct or pattern search methods.  Direct search
\citep{kolda2003optimization} makes use
of the well-known fact that if the objective function is continuously
differentiable, then at least one of the directions of any
\emph{positive spanning set} (a set that spans the space with
non-negative coefficients, abbreviated as PSS in what follows) is a
descent direction. It explores the space by evaluating the function at new points that are located in a
number of predefined search-directions from the current iterate, at a
distance from that iterate that varies with time. The algorithm moves to a new iterate
only if this iterate yields a sufficient improvement of the value of the
function (there exist other versions of the algorithm where the
sufficient decrease condition is replaced with a constraint on the
choice of the trial directions).  The sample-complexity of such an algorithm
has been analyzed in \citep{vicente2013worst} in the deterministic and
unconstrained setting. 
\cite{lewis2000pattern} study direct search in linearly
constrained domains. Handling the constraints in direct search is
quite simple, as it consists in testing only the directions in the set
of search-directions that are feasible.  \cite{gratton2015direct}
analyzed direct search with random sets of search-directions instead
of predefined ones and later extended the analysis to the case of
linearly-constrained domains
\citep{gratton2019direct}. \cite{dzahini2022expected} extended their
work by analyzing a similar algorithm in the presence of noise, but
without constraints.  \cite{dzahini2022expected} relies on an
assumption on the decrease of the noise. In this paper, we will study
direct search algorithms that rely on a number of samples at each
point to build tight estimates of the function at the trial points,
which can be understood as a way to decrease the
noise. \cite{dzahini2022expected} analyzes some notion of sample
complexity of direct search, which only takes the iterates into
account rather than the number of function evaluations needed, which
is not appropriate in our setting.

Our purpose is to study these methods that are suitable for the blind
resource allocation model, i.e. in particular, compatible with zeroth-order feedback,
computationally tractable and that do not require to sample points outside of the feasible domain.
Besides satisfying the above requirements,
these algorithms have the advantage of being approximate descent algorithms with high probability,
a guarantee that is useful in practice, allowing, for instance, warm start from previously tested allocations. 
The adaptation of direct search to the noisy case is achieved by performing enough sampling to
ensure that the algorithm moves to a new iterate only if it results in
a sufficient improvement, with high probability. We propose two ways
of doing so: the first method (termed FDS-Plan) simply computes the
number of necessary evaluations ahead of time, whereas the second one,
FDS-Seq, uses a sequential testing strategy to interrupt sampling as
early as possible.
The algorithms are specified in Section \ref{sec:algorithms}. An illustration of the behavior of the proposed algorithms can be found in this same section, alongside an illustration of other baseline strategies, which allows for understanding the specifics of direct search. We
analyze the cumulative regret of these algorithms in Section
\ref{sec:analysis}, providing an upper bound of their regret of the
order of $T^{2/3}$ (up to logarithmic factors), when the optimum is in
the interior of the feasible domain. A significant technical challenge for the analysis in terms of regret is that, while in traditional analyses of direct-search, the number of rounds is fixed and the analysis proceeds by looking at the distance to the optimum at each round, here, the number of rounds is random (the indexing of the regret is the actual number of function evaluations instead of the number of rounds). We start Section
\ref{sec:analysis} by the simpler case in which there is neither noise
nor constraints, showing that in this basic setup the regret of direct
search is bounded by a constant.\looseness-1

\section{Algorithms} \label{sec:algorithms}
\subsection{Description of the Algorithms}
In Algorithm \ref{alg:dir_search} below, we start by describing the most
common version of the direct search method used for deterministic and
unconstrained optimization. It requires the setting of an initial
point $x_0$ and an initial parameter $\alpha_0$. The learner also
specifies a PSS $\mathbb{D}$, that is, a set of directions that
spans $\mathbb{R}^d$ with non-negative coefficients.
At each iteration, the algorithm sequentially tests
points at a distance $\alpha_k$ from the current iterate and in the
directions defined by $\mathbb{D}$. If none of the test points results
in a sufficient decrease of the function's value, the iteration is
declared unsuccessful and the trial radius $\alpha_k$ shrinks by a
factor $\theta<1$, otherwise, the iteration is declared successful
and the iterate $x_k$ is moved to the first trial point that results
in a sufficient improvement. A decrease is considered to be sufficient
if it is larger than some predefined forcing function of $\alpha_k$,
that we take here to be quadratic, with a coefficient that can be set
by the learner. \looseness-1

\begin{algorithm}
Choose $x_0\in\rn$, $\alpha_0>0$, $\theta< 1$, $c>0$, $\rho(u)= c u^2$ and a PSS $\mathbb{D}$\\
\For {$k = 0 \ldots K$}{
  Set \textit{UnsuccessfulSearch} $\gets$ True\\
  \For {$v \in  \mathbb{D}$}{
    Evaluate $f(x_k + \alpha_k v)$\\
    \If {$f(x_k)- f(x_k+\alpha_k v) \geq  \rho(\alpha_k)$}{
      Set $x_{k+1}\gets x_k+ \alpha_k v$ and $\alpha_{k+1}\gets\alpha_k$\\
      Set \textit{UnsuccessfulSearch} $\gets$ False\\
      \Break
    }
  }
  \If{UnsuccessfulSearch}{
    Set $x_{k+1}\gets x_k$ and $\alpha_{k+1}\gets \theta\alpha_k$
  }
}
\caption{Direct Search with sufficient decrease}\label{alg:dir_search}
\end{algorithm}

The analysis can also be adapted to the presence of a growth factor
$\phi \geq 1$ by which the trial radius $\alpha_k$ expands at
successful iterations. For simplicity, we choose to focus on the case where $\phi  = 1$, as this parameter does not modify the regret rates obtained
in Section~\ref{sec:analysis}. Also note that Algorithm \ref{alg:dir_search} is a descent algorithm with respect to the iterates $x_k$, i.e., the sequence
$(f(x_k))_k$ is decreasing.

Obviously, different choices of PSS result in different
trajectories of the algorithm.  Setting $\mathbb{D}$ as the
set of $2d$ vectors of the positive and negative coordinate directions
results in the algorithm known as coordinate or compass search. Other
frequently considered choices include random directions, as in
\citep{gratton2015direct,gratton2019direct,dzahini2022expected}. \looseness-1

\begin{minipage}{0.49\textwidth}
\vspace{-8.em}
\begin{algorithm}[H]
\caption{FDS-Plan}
Choose $x_0\in\rn$, $\alpha_0>0$, $\theta< 1$, $c>0$, $\rho(u)= c u^2$\\
\For {$k = 0 \ldots K$}{
  Set \textit{UnsuccessfulSearch} $\gets$ True\\
  Select a set of directions $\mathbb{D}_k$ \\
  Set $N_k= \frac{32\sigma^2 \log (2/\delta)}{\rho(\alpha_k)^{2}}$.\\
  Estimate $f(x_k)$ by making $N_k$ samples at $x_k$ and setting $\hat{f}(x_k)=  \frac{1}{N_k}\sum_{j=1}^{N_k}f(x_k)+ \epsilon_j$\\
  \For {$v \in  \mathbb{D}_k$ such that $x_k + \alpha_k v \in \mathcal{D}$}{
    Estimate  $f(x_k + \alpha_k v)$ by
    making $N_k$ samples at $x_k+ \alpha_k v$ and setting $\hat{f}(x_k+ \alpha_k v)=  \frac{1}{N_k}\sum_{j=1}^{N_k}f(x_k+ \alpha_k v)+ \epsilon_j'$ \\
    \If {$\hat f(x_k)- \hat f(x_k+\alpha_k v) \geq  \rho(\alpha_k)$}{
      Set $x_{k+1}\gets x_k+ \alpha_k v$ and $\alpha_{k+1}\gets\alpha_k$\\
      Set \textit{UnsuccessfulSearch} $\gets$ False\\
      \Break
    }
  }
  \If{UnsuccessfulSearch}{
    Set $x_{k+1}\gets x_k$ and $\alpha_{k+1}\gets \theta\alpha_k$
  }
}
\label{alg:con_dir_search_f_samples}
\end{algorithm} 
\end{minipage}  \hfill
\begin{minipage}{0.49\textwidth}
\begin{algorithm}[H]
\caption{FDS-Seq}
Choose $x_0\in\rn$, $\alpha_0>0$, $\delta>0$, $\theta < 1$, $c>0$,  $\rho(u)= c u^2$\\
		Set \textit{UnsuccessfulSearch} $\gets$ True\\
		\For {$k = 0 \ldots K$}{
    Select a set of directions $\mathbb{D}_k$ \\
	 \For{$v$ in  $\mathbb{D}_k$ such that $x_k + \alpha_k v \in \mathcal{D}$}{ 
		 \While{Condition \ref{eq:cond_test} is not satisfied 
	 }{ 
	 \If{$n_{v,k}\leq n_{0,k}$}{Sample at $x_k + \alpha_k v$ and update $n_{v,k}$ and the empirical mean $\hat{f}_{n_{v,k}}(x_k + \alpha_k v)$.} 
	 \Else{Sample at $x_k $ and update $n_{0,k}$ and  the empirical mean $\hat{f}_{n_{0,k}}(x_k )$
	 }}
	 \If{$\hat{f}_{n_{0,k}}(x_k)- \hat{f}_{n_{v,k}}(x_k+\alpha_k v) \geq  \rho(\alpha_k)$
	 }{Set  $x_{k+1}\gets x_k+ \alpha_k v$, and $\alpha_{k+1}\gets  \alpha_k$.
	 \\
	 \textit{UnsuccessfulSearch} $\gets$ False\\
	 \textbf{Break.}}
	 }
	\If {UnsuccessfulSearch} {
	Set $x_{k+1}\gets x_k$ and $\alpha_{k+1}\gets \theta\alpha_k$ 
	}
	}
	\label{alg:f_dir_search_f_tests}
\end{algorithm}
\end{minipage}

We apply three sorts of modifications to Algorithm \ref{alg:dir_search}
in order to adapt it to the more general model introduced in Section
\ref{sec:model}.  The first one consists in sampling a trial
point only if it is feasible. The second consists in allowing changes
in the set of directions $\mathbb{D}_k$ considered. This is to account for the fact 
that the change in search-radius at every round impacts the set of admissible directions, denoted $\mathcal{A}_k$.
We thus only need to sample directions that span $\mathcal{A}_k$ positively,
and not the whole optimization domain $\mathcal{D}$. The third modification consists in
introducing estimation stages that allow building reliable estimates
of $f$ at the trial points. We propose two ways of doing so, that
result in two different algorithms.  The first algorithm that we study
is a plug-in version of Algorithm \ref{alg:dir_search} in which we
replace $f(x_k)$ and $f(x_k+ \alpha_k v)$ with their empirical estimates,
consisting of means computed from
$N_k= \frac{32\sigma^2 \log (2/\delta)}{\rho(\alpha_k)^{2}}$ samples.
This number of samples guarantees that with high probability, the
estimation gap is smaller than $\rho(\alpha_k)/4$, which in turn
ensures that an iteration is declared successful only when it leads to
a decrease of $f(x_k)$ by at least $\rho(\alpha_k)/2$ and that an
unsuccessful iteration cannot occur if there exists a direction $v$ in
$\mathbb{D}_k$ such that the decrease achieved by moving to
$x_k + \alpha_k v $ is larger than $3\rho(\alpha_k)/2$.  The resulting
algorithm is termed Feasible Direct Search with a planned number of
samples (FDS-Plan) and described in Algorithm
\ref{alg:con_dir_search_f_samples}.
We also propose a faster algorithm, Feasible Direct Search with
Sequential Tests (FDS-Seq) described in Algorithm
\ref{alg:f_dir_search_f_tests}. For any $v\in\mathbb{D}_k$, instead of
planning the number of samples at $x_k + \alpha_k v$ ahead of time, it
samples at $x_k+ \alpha_k v$ and $x_k$ until either
\begin{equation}\label{eq:cond_test}
\begin{cases} \left|
 \hat{f}_{n_{0,k}}(x_k) - \hat{f}_{n_{v,k}}(x_k+ \alpha_k v) - \rho(\alpha_k) 
 \right| 
 \leq 
\sqrt{
2 \sigma^2\log(1/\delta)\Big(
\frac{1}{n_{0,k}}+ \frac{1}{n_{v,k}}
\Big)}\;,\\
\text{or } \left( n_{0,k} \geq N_k \text{ and } n_{v,k} \geq N_k \right)\;,
\end{cases}
\end{equation}
$n_{0,k}$ and $n_{v,k}$ denoting the number of samples at $x_k$
and $x_k + \alpha_k v$ and $\hat{f}_{n_{0,k}}(x_k)$ and
$\hat{f}_{n_{v,k}}(x_k + \alpha_k v)$  the resulting empirical
means. Successful and unsuccessful iterations are defined as in
FDS-Plan and trigger the same actions.\looseness-1

The sequential stopping rule is designed to achieve early detection of
sufficient decrease, but also to detect as early as possible the cases
in which the trial point cannot lead to a sufficient decrease.  The
first test in Condition \ref{eq:cond_test} is a consequence of the fact that
the estimation gap at $x_k$ (respectively $x_k + \alpha_k v$) is
$\sigma^2/ n_{0,k}$ subgaussian (respectively $\sigma^2/ n_{v,k}$
subgaussian). The second test of Condition \ref{eq:cond_test}
corresponds to a safeguard preventing from waiting too long when
the decrease induced by the trial point is very close to
the sufficiency threshold.  At worst, the number of evaluations needed
is the same as in Algorithm \ref{alg:con_dir_search_f_samples}.
Essentially, the sequential stopping rule reduces the number of
evaluations needed for each iteration but maintains the desirable
property that with high probability, an iteration is declared
successful only if it leads to a decrease of at least
$\rho(\alpha_k)/2$ and that an iteration cannot be declared
unsuccessful if there exists a direction $v$ in $\mathbb{D}_k$ such that
the decrease achieved by moving to $x_k + \alpha_k v $ is larger than
$3\rho(\alpha_k)/2$.

\subsection{Illustration}
\label{sec:illustration}

 \begin{figure}[ht]
\centering
\subfigure[FDS-Plan]{
\includegraphics[width=0.23\linewidth]{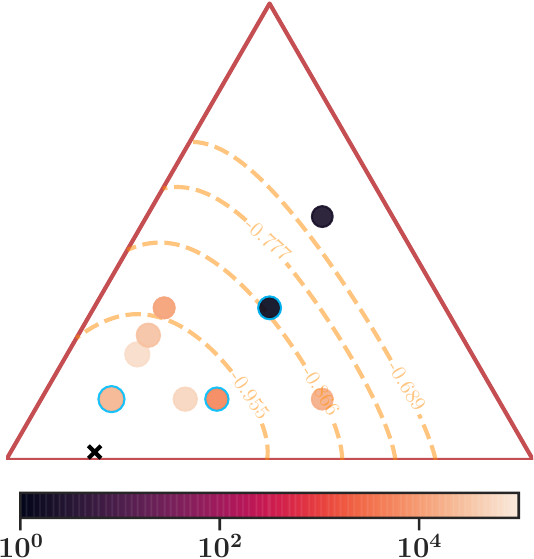}
\label{fig:FDS-Plan_2}
}\quad
\subfigure[HOO]{%
\includegraphics[width=0.23\linewidth]{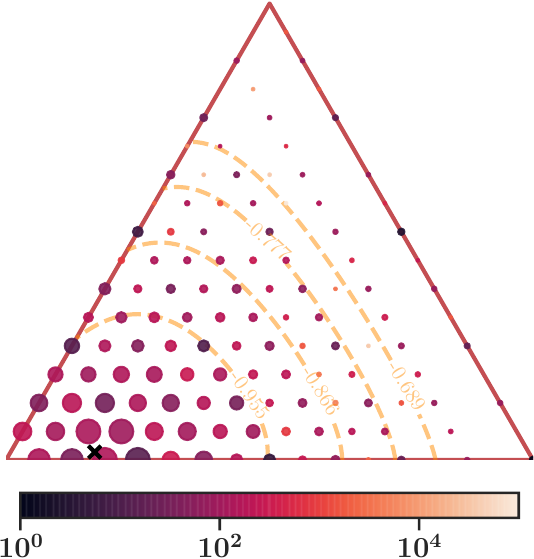}
\label{fig:HOO}
} \quad
\subfigure[Gradient Descent with two-points estimate]{%
\includegraphics[width=0.23\linewidth]{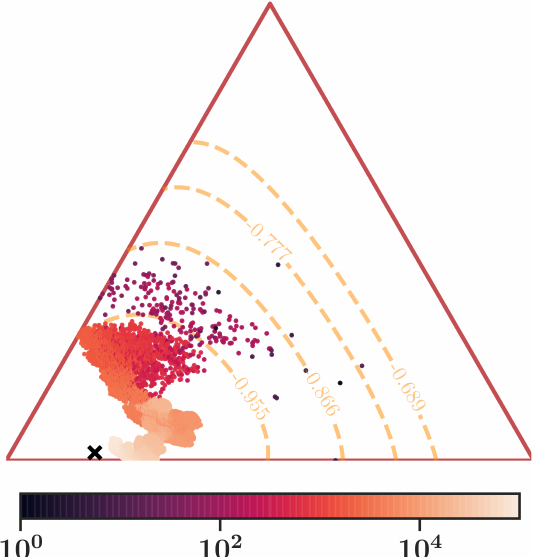}
\label{fig:GD}
} \\
\subfigure[FDS-Seq]{%
\includegraphics[width=0.23\linewidth]{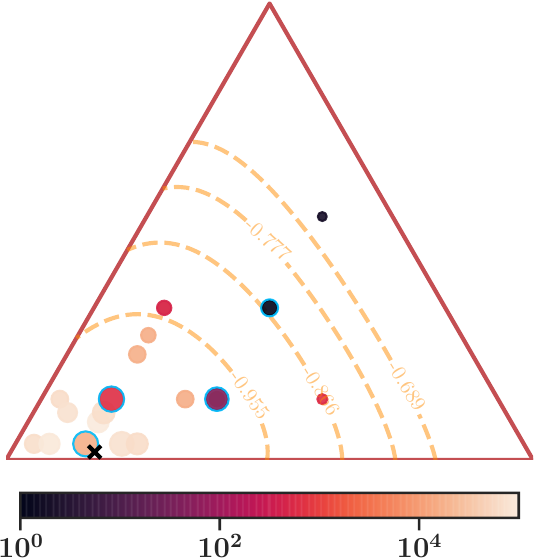}
\label{fig:FDS-Seq_2}}
\quad
\subfigure[UCB]{%
\includegraphics[width=0.23\linewidth]{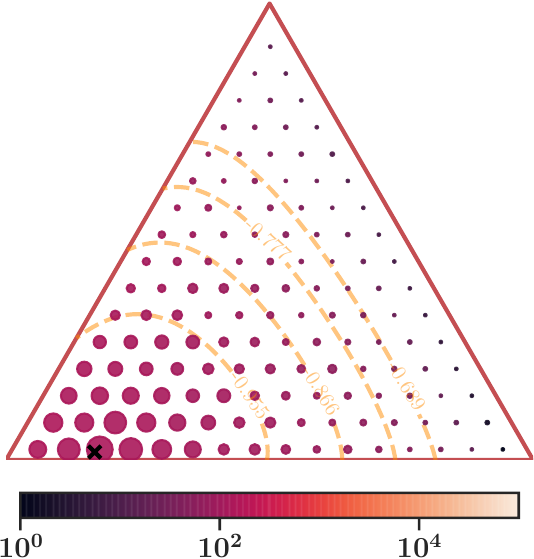}
\label{fig:UCB}}
\quad
\subfigure[Gradient Descent with one-point estimate]{%
\includegraphics[width=0.23\linewidth]{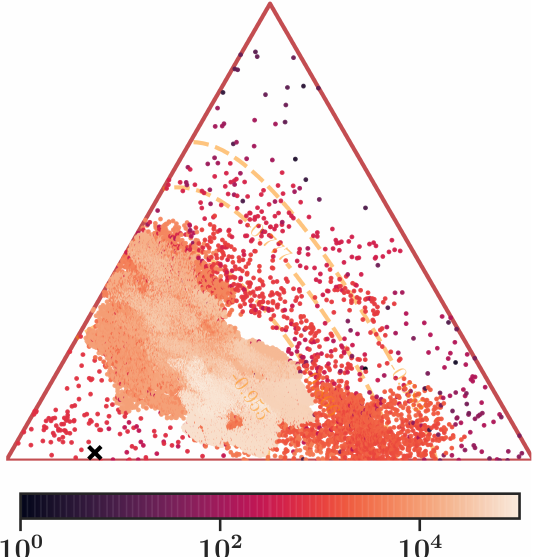}
\label{fig:GD_2}
}
\caption{Single trajectories}
\label{fig:trajectories}
\end{figure}

In order to illustrate graphically the behavior of the proposed
methods, we show on Figure \ref{fig:FDS-Plan_2} and
\ref{fig:FDS-Seq_2} their trajectories in the case where there are
3 resources ($d=2$) and the loss functions associated to each resource
$i \in \{1, \dots,d+1\}$ are of the form
$w_i(x)= -\tau_i \frac{\log(1+ \gamma x)} {\log(1+ \gamma )}$ with
$\gamma=2$, $\tau_1= 1$, $\tau_2= 0.45$, and $\tau_3=0.95$.
We set the horizon to $T=100,000$ and use a Gaussian noise with
standard deviation $\sigma= 0.1$, a realistic value for budget
allocation problems.  In the symmetric representation of Figure
\ref{fig:trajectories}, the three vertices correspond to the points
where one of the resource is fully saturated (equal to 1) and the
edges correspond to linear paths along which one of the resources is set to 0.
The contour lines of the target function are materialized by orange
lines and the location of the minimum is marked by a black cross.  The
size of each point is a logarithmically growing function of the number
of samples made at this point, and its color is a function of the
index of the first round at which it has been sampled.  Finally,
points at which a successful iteration of FDS-Plan (and
FDS-Seq) occurred are circled in blue. The parameters of both versions
of feasible direct search are $\alpha_0=0.2$, $c=5$ and $\theta=0.7$,
and the initial point corresponds to the allocation
$x(0)=(1/3, 1/3, 1/3)$ (center of the simplex). To make the figure more interpretable, we choose to set a fixed $\mathbb{D}$. The set of directions
chosen for these algorithms are the 6 directions that support the
edges of the simplex (in both directions).

In a first phase, the algorithm proceeds rapidly by testing
directions until a sufficient descent direction is found. Afterwards,
when the iterates get closer to the minimizer,
the search area iteratively shrinks as finding descent directions becomes harder. In the first phase, the
trajectory is similar to a descent path that would result from a
gradient descent algorithm 
while the second phase is closer to
the behavior of bandit algorithms based on hierarchical partitions, like HOO
\citep{bubeck2011x}. In order to illustrate the
differences with such algorithms, we also plot the
trajectories of baseline methods, either related to gradient descent
or bandits with hierarchical partitioning.
Before turning to these other algorithms, it is important to note the
difference between FDS-Plan and FDS-Seq. Figure \ref{fig:FDS-Seq_2}
shows that FDS-Seq is faster than FDS-Plan, as it spends less time on
the first iterations, in which it is easy to determine whether the
trial points lead to a sufficient decrease. The FDS-Seq algorithm can thus
perform more iterations than FDS-Plan.
A common point of these two algorithms that is illustrated on the figure is that they are approximate-descent algorithms with high probability, an interesting quality for practitioners interested in interpretability.


Let us now comment on Figure \ref{fig:HOO}, that represents the
trajectory of a version of HOO. It is not straightforward to apply
algorithms for $\mathcal{X}-$armed bandits like HOO on the simplex,
because it implies constructing balanced hierarchical partitions of
the simplex.  
We thus explain our implementation of HOO in Appendix \ref{app:HOO}.
On Figure \ref{fig:HOO}, we observe that this algorithm explores
the partition tree in a way that favors the cells close to the
optimum while persistently visiting cells that are clearly far
from the minimizer.
The behavior of algorithms based on direct search can thus be preferred because it makes
warm-start possible, in the sense that prior belief on the location of the
minimizer can be used for setting the initial point, which is not possible
for HOO. The fact that suboptimal points will keep being
sampled until the end of the experiment can also be difficult to
accept for practitioners such as advertisers for example.

We also illustrate the behavior of UCB on a discretization of the
space, which is an interesting strategy, especially in dimension 2.
The discretization used for UCB consists of points arranged in a
regular grid of $[0,1]^2$, from which the points lying outside of the
feasible domain have been removed. The step parameter of the grid is
taken as $T^{-1/4}$ as suggested by~\cite{combes2014unimodal}.
Although simpler than HOO, this algorithm results in similar sampling
patterns, as shown on Figure \ref{fig:UCB}, and
hence shares some of its drawbacks. The performance of UCB is good in
dimension 2, as the regret can be proved to be of the order of
$\sqrt{T}$ with this choice of step-size (see
\cite{combes2014unimodal}) but it will worsen in higher dimension due
to the difficulty of simultaneously controlling the distance between
grid points and the overall number of points in the grid. In fact, the
optimal step in this case is of the order of $T^{-1/(d+2)}$ and the
resulting regret is of the order of $T^{d/d+2}$,
 which only works in favor of UCB for small values of $d$.
Note that  while \cite{combes2014unimodal} also provide weaker regret guarantees under more general assumptions, the assumptions required to obtain this order of regret are similar to Assumption \ref{ass:reg}. They are only less constraining than Assumption \ref{ass:reg} in that they require a quadratic upper and lower bound on the function locally near the optimum, whereas Assumption \ref{ass:reg} should hold uniformly on the domain.

Lastly, we illustrate the behavior of two methods related to
stochastic gradient descent. The first method is related to that
proposed by \cite{akhavan2020exploiting}. Without constraints, this
method would estimate the gradient of the function at $x_t$, by
evaluating the function at $y^+_t = x_t + h_t Z_t$ and
$y^-_t = x_t - h_t Z_t$, where $Z_t$ is a random vector of the sphere
of radius $1$, and use $\frac{f(y^+_t)- f(y^-_t)}{h_t}Z_t$ as an
estimation of the gradient. The method in itself does not take
constraints into account, but a slight modification results in an
algorithm that is feasible in the presence of constraints. This
modification consists in performing a homothetic perturbation
\citep{bravo2018bandit} on the evaluation points $y_t^+$ and $y_t^-$:
instead of using these points, the algorithm evaluates the function at
$\tilde{y_t}= y_t + h_t/r (c-x_t)$, where $c$ is a point in the
interior of the simplex such that
$\mathcal{B}(c,r)\subset \mathcal{D}$. We use
$\tilde{y}^+_t = y^+_t + h_t/r (c-x_t)$ and
$\tilde{y}^-_t = y^-_t - h_t/r (c-x_t)$, where $c$ is a point in the
interior of the simplex such that
$\mathcal{B}(c,r)\subset \mathcal{D}$.  This ensures that the
evaluation point belongs to the constrained domain, provided that
$x_t \in \mathcal{D}$, but adds a bias which is proportional to $h_t$,
under suitable regularity assumptions on $f$. To ensure that
$x_t \in \mathcal{D}$, we also project the result of the gradient
descent step on $\mathcal{D}$.  We use an estimation step $h_t$ equal
to $(t/2)^{-1/3}$ and a learning rate decreasing as $1/(2.5 t)$. This
choice is justified by the following reasoning: with this choice of
value for the learning rate and $h_t$ set to $t^{-1/4}$,
$\sum_{t=1}^{T} f(x_t)-f(x^*)$ would be bounded by $O(T^{1/2})$, thanks
to the analysis of \cite{akhavan2020exploiting}; but we have to add a
term related to the sum of evaluation steps $\sum_{t=1}^{T/2}(f(y^+_t)+f(y^-_t)-2f(x^*))$ to bound the actual regret, which under suitable assumptions is of order $\sum_{t=1}^T h_t$; so that setting $h_t$ to
$(t/2)^{-1/3}$ allows to bound both terms by $O(T^{2/3})$. On Figure
\ref{fig:GD}, we see one trajectory of this method when the starting
iterate is on the center of the simplex. The convergence speed is rather fast
at the beginning but the speed is limited by the homothetic
perturbation.\looseness-1
 
The second method is inspired by \cite{flaxman2004online}. This paper
proposes to use gradient descent with a one-point gradient
estimation. In order to evaluate the gradient of the function at
$x_t$, the algorithm evaluates the function at $y_t = x_t + h_t Z_t$,
where $Z_t$ is a random vector of the sphere of radius $1$, and uses
$f(y_t)/h_t z_t$ as an estimation of the gradient. As for the previous
method, we apply an homothetic perturbation to $y_t$.  We use an
estimation step $h_t$ equal to $t^{-1/3}$ and a learning rate
decreasing as $1/(2.5 t)$. The trajectory that we see on Figure
\ref{fig:GD_2} is not very indicative of the average performance of
the algorithm, since this method comes with a very high variance. We
see that the algorithm generates a trajectory that roughly gets closer to the
minimizer, but that is far from being a descent path because of the
poor estimation of the gradient. This method, which has been designed
for adversarially evolving objective functions, is clearly not
advisable for static objectives with stochastic perturbations.

\section{Regret Analysis}
\label{sec:analysis}
As discussed in the introduction, the regret criterion takes into
account the number $T$ of function evaluations instead of focusing on
the number $K$ of iterations, as in the more traditional analysis. In
the noiseless case of Algorithm~\ref{alg:dir_search}, $T$ and $K$
differ by a factor of at most $|\mathbb{D}|+1$ and this is not an
issue. However, in the case of
Algorithms~\ref{alg:con_dir_search_f_samples}
and~\ref{alg:f_dir_search_f_tests}, the situation is very different as
the number of function evaluations per iteration is stochastic and
typically increases as the algorithm converges. In this case, it is
not possible to predict in advance the evolution of $T$ as a function
of $K$ because it depends on the function and starting point. A
significant part of the analysis is indeed devoted to quantifying this
phenomenon. In practice, it means that in order to comply with a
 number $T$ of function evaluations set in advance, the algorithms are run
without a fixed number of rounds $K$, instead they are run until the
number of function evaluations reaches $T$.

In the following, we  analyze the proposed algorithms and show
that in the constrained and noisy set-up of interest, FDS-Plan and
FDS-Seq have a regret of the order of $(\log T)^{2/3} T^{2/3}$ under
some further assumptions on $f$ and on the chosen direction set
$\mathbb{D}_k$, provided that the optimal point lies in the interior of
the feasible domain.
To provide some intuition on the proofs, we start with the analysis of
Algorithm \ref{alg:dir_search} in the unconstrained and deterministic
setting, thereby providing the first regret bound of any direct search
algorithm.\looseness-1

\subsection{Warm-up: the Unconstrained and Deterministic Setting} 
The choice of $\mathbb{D}$ is decisive for the performance of 
both FDS-Plan and FDS-Seq.
In the sequel, we make the following assumption
on $\mathbb{D}$.\looseness-1

\begin{assumption}\label{ass:kappa} The vectors of $\mathbb{D}$ have unit norm and the cosine measure of $\mathbb{D}$ is lower-bounded, i.e, there exists $\kappa>0$ such that
	$$cm(\mathbb{D}):= \min_{u \in \rn, u\neq 0} \max_{v \in \mathbb{D}} \frac{u^T v}{\|u\| \|v\|} > \kappa\;.$$ 
\end{assumption}

Assumption \ref{ass:kappa}, common in direct search's literature, guarantees 
that at each
iteration $k$, the cosine similarity between at least one direction in
$\mathbb{D}$ and $-\nabla f (x_k)$ is larger than $\kappa$.
If $\mathbb{D}$
is a PSS there exists a $\kappa$ satisfying it.\looseness-1
 
\begin{theorem}\label{th:direct_search_reg_bound}
  Under Assumptions \ref{ass:reg} and \ref{ass:kappa}, the cumulative
  regret of Algorithm \ref{alg:dir_search} satisfies
  \begin{align*}
  R_T \leq & \frac{|\mathbb{D}|+1}{c}\Bigg[\left(\frac{1}{1- \theta^2}\left(f(x_{0}) - f(x_\star) +\rho(\alpha_0)\right)\right)\left(\left(1+ \frac{\eta}{a}\right) \eta + \beta \right) \\
                    &+(f(x_{0}) - f(x_\star))
                      \left(\frac{\beta}{a \alpha_0}\|\nabla f(x_0) \| 
                      + \beta \right)\Bigg]\;,
  \end{align*}
  where
  $\eta :=\frac{\beta}{a}\frac{1}{\kappa \theta} (c + \frac{\beta}{2})
  $.
\end{theorem}

This result shows that under Assumption \ref{ass:reg}, the asymptotic
behavior of the regret of direct search can be compared to that of the
more traditional gradient descent algorithm, whose regret is also
bounded under this assumption \citep[see Theorem 3.6
of][]{bubeck2008online}.
 

\subsubsection{Elements of Proof}
The proof of Theorem \ref{th:direct_search_reg_bound} combines two
well-known properties of direct search and the following lemma
holding for any descent algorithm. %
\begin{lemma}\label{lem:bound_increase_grad} If $f$ satisfies Assumption \ref{ass:reg}, then
$\displaystyle{\forall k'>k, \|\nabla f(x_{k'}) \|\leq \frac{\beta}{a}\|\nabla f(x_{k}) \|}$.
\end{lemma}
In the proof of Theorem \ref{th:direct_search_reg_bound}, Lemma
\ref{lem:bound_increase_grad} is used in conjunction with the
following well-known property \citep[see e.g ][]{vicente2013worst} of
direct search.

\begin{lemma}\label{lem:gradnorm_failures}
  If $f$ satisfies Assumptions \ref{ass:reg} and \ref{ass:kappa} and
  iteration $k$ corresponds to an unsuccessful iteration, then
  $\displaystyle{\|\nabla f(x_k) \|\leq
    \frac{1}{\kappa}\left(\frac{\beta}{2} \alpha_k +
      \frac{\rho(\alpha_k)}{\alpha_k}\right)=
    \frac{1}{\kappa}\left(\frac{\beta}{2} + c \right) \alpha_k.}$
\end{lemma}

The above lemma follows from the definition of the cosine measure of
$\mathbb{D}$, that results in a bound of $v_k^T (-\nabla f(x_k))$,
where $v_k$ is the direction in $\mathbb{D}$ maximizing the gap
$f(x_k)- f(x_k + \alpha_k v)$, and from the smoothness assumption on
$f$.  Thanks to Lemma \ref{lem:bound_increase_grad}, this lemma also
means that when iteration $k$ is unsuccessful, we can bound all
subsequent gradients by
$\frac{\beta}{\kappa a}(\frac{\beta}{2} + c ) \alpha_k$.  We can
deduce that for any $k'$ following the first unsuccessful iteration,
$\| \nabla f(x_{k'}) \| \leq \eta \alpha_{k'}$, where
$\eta :=\frac{\beta}{a}\frac{1}{\kappa \theta} (c + \frac{\beta}{2})
$.  Indeed, if $k'$ is the index of an unsuccessful iteration, Lemma
\ref{lem:bound_increase_grad} suffices to prove
$\| \nabla f(x_{k'}) \| \leq \eta \theta \alpha_{k'}$.  In contrast,
when $k'$ is the index of a successful iteration, one should consider
the last unsuccessful iteration $k$.  Since
$\alpha_{k'} \geq \theta \alpha_{k}$, there has been at most one
reduction of the step-size since iteration $k$
and \begin{equation}\label{eq:bound_increase_grad_unif}\|\nabla
  f(x_{k'})\|\leq \frac{\beta}{\kappa a \theta}\left(\frac{\beta}{2} +
    c \right) \alpha_{k'} = \eta \alpha_{k'}\;.\end{equation} The
following general argument on direct search is the final key element
of the proof of Theorem \ref{th:direct_search_reg_bound}, that links
the sum of the squared search-radius to the initial sub-optimality gap
$f(x_{0}) - f(x_\star)$.

\begin{lemma}\label{lem:sumsteps} If $f$ satisfies Assumptions \ref{ass:reg} and \ref{ass:kappa},
$$ \sum_{k=0}^{\infty}\rho(\alpha_k)= \sum_{k=0}^{\infty} c\alpha_k^2\leq \frac{1}{1- \theta^2}(f(x_{0}) - f(x_\star) +\rho(\alpha_0))\;. $$
\end{lemma}

Lemma \ref{lem:sumsteps} can be explained by the fact that
$\rho(\alpha_k)$ decreases geometrically by a ratio $\theta^2$ between
two successive successful iterations, so that the contribution to the
sum of these iterations boils down to multiplying the remainder of the
sum by a factor of $\frac{1}{1- \theta^2}$. The sum on successful
iterations cannot be too large, because by definition of successful
iterations, $f(x_k)- f(x_\star)$ is lower bounded by this sum plus the
initial sub-optimality gap $f(x_{0}) - f(x_\star)$.  Bringing Lemma
\ref{lem:sumsteps} and the bound of
\Eqref{eq:bound_increase_grad_unif} together results in a bound on the
squared norm of the gradients $\|\nabla f(x_{k})\|^2$ after the first
unsuccessful iteration. Using the regularity conditions of Assumption
\ref{ass:reg} suffices to relate the regret to the squared norm of the
gradients, which in turn yields Theorem
\ref{th:direct_search_reg_bound}. The complete proof can be found in
Appendix \ref{sec:app_det_unc}.

\subsection{The Constrained and Noisy Setting}
We now turn to the noisy and constrained case described in Section \ref{sec:model}.
We further impose the following assumptions on the domain. 
\begin{assumption}\label{ass:bounded_D}
$\D$ is contained in a ball of radius $b$.
\end{assumption}

This assumption, together with assumption \ref{ass:reg}, implies that
$f(x) -f(x_\star)$ is bounded.  
It also follows from these two assumptions that
$\nabla f$ is bounded in norm by a constant, denoted by $B$, on   the feasible set $\D$.

While in the unconstrained case, the chosen PSS $\mathbb{D}$ only
needed to satisfy Assumption \ref{ass:kappa}, a stronger assumption is
required in the presence of linear constraints. Indeed, Assumption
\ref{ass:kappa} was a way to ensure that there was at least one trial
direction $v$ in $\mathbb{D}$ satisfying
$\frac{- \nabla f(x_k)^Tv}{\|v\| \|\nabla{f}(x_k) \|}\geq \kappa$.
This property is not sufficient in the constrained case, because in
this case, the directions of interest at iteration $k$ in $\mathbb{D}_k$
are those that are feasible. A problem that might arise for example,
is that a sufficient descent direction is not detected even in a
situation where $x_k - \alpha_k \nabla f(x_k)$ is feasible, because
the set of feasible directions in $\mathbb{D}_k$ does not positively
span the feasible region.  To avoid such cases, we impose a constraint
on $\mathbb{D}_k$ that involves the notion of approximate tangent
cones. Approximate tangent cones at a point $x$ are the polar cones of
the cones that are generated by the $\alpha$-binding constraints at
$x$ as defined by \cite{kolda2007stationarity} (see in particular Figure 
2.1 of \cite{kolda2007stationarity} for an illustration of the notion
of approximate tangent cones).

Let $a_i^T$ be the $i$-th row of the constraint matrix $A_I$ and let
$\mathcal{C}_i = \{y, \text{ s.t}~ a_i^T y=u_i\}$ denote the sets
where the $i$-th constraint are binding. If there exists a point of
$\mathcal{C}_i$ at a distance smaller than $\alpha$ from $x$, then the
$i$-th constraint is said to be $\alpha$-binding. The indices of
$\alpha$-binding constraints at $x$ are denoted
$I (x, \alpha)= \{i, ~ \mathrm{dist}(x,\mathcal{C}_i) \leq \alpha \},$
where $\mathrm{dist}$ is induced by the Euclidian distance. We define
the approximate normal cone $N(x,\alpha)$ to be the cone generated by
the set $\{a_i, ~ \text{s.t}. ~ i \in I(x, \alpha)\}\cup \{0\}$. The
approximate tangent cone $T(x,\alpha)$ is the polar of $N(x,\alpha)$ ,
which means that
$\displaystyle{T(x,\alpha) = \big\{v: y^Tv \leq 0 , ~\forall y \in
  N(x,\alpha) \big\}}.$ Informally $T(x,\alpha)$ is the cone inside of
the boundaries generated by the $\alpha$-binding constraints at
$x$. 
We highlight that since the number of constraints $m$ is finite, there
can only be a finite number, smaller than $2^m$, of tangent
cones. Consequently, Assumption \ref{ass:gk} is rather mild.

\begin{assumption}\label{ass:gk}
For $k \in \{1 \ldots K\}$, $\mathbb{D}_k$ contains a set $\mathcal{G}_k$ of positively generating directions of $T(x, \alpha)$ included in $T(x, \alpha)$, for any $x \in \D$ and $\alpha \in \R_{+}$.
\end{assumption}

 In the following, we denote by $\mathcal{G}_k$ such a set. Assumption \ref{ass:gk} was already necessary in \citep{kolda2003optimization}
and \cite{lewis2000pattern}, while in
\citep{gratton2019direct}, the descent set at iteration $k$ is assumed
to be contained in $T(x, \alpha)$ and to generate it. We explain the
purpose of this assumption in the following. While the purpose of
Assumption \ref{ass:kappa} was to ensure that the maximal cosine
similarity of a vector in $\mathbb{D}_k$ with $- \nabla{f}(x_k)$ was
bounded away from $0$, we focus on a different kind of measure of
similarity to $- \nabla{f}(x_k)$ defined as\looseness-1
$$\begin{cases}
\max_{v \in \mathcal{G}_k} \frac{- \nabla{f}(x_k)^T v}{\|P_{T(x, \alpha)}(-\nabla{f}(x_k))\| \|v\|} \text{ if }P_{T(x, \alpha)}(- \nabla{f}(x_k))\neq 0 \;,\\ 1 \text{ otherwise.}
\end{cases}$$ If $P_{T(x, \alpha)}(- \nabla{f}(x_k))$ gets close to
$0$, this measure of similarity to $- \nabla{f}(x_k)$ does not necessarily
become small, although $- \nabla{f}(x_k)^T v$ is small for any $v$ in
$\mathcal{G}_k$.  In order to bound this measure of similarity between
a vector of $\mathcal{G}_k$ and $-\nabla f (x_k)$ away from $0$, we
define the following approximate cosine measure:
 $$ cm_{T(x_k, \alpha_k)}(\mathcal{G}_k):= \inf_{u \in \rn,  P_{T(x_k, \alpha_k)}(u)\neq 0} \max_{v \in \mathcal{G}_k} \frac{u^T v}{\|P_{T(x_k, \alpha_k)}(u)\| \|v\|}\;.$$ 
As proved by \cite{lewis2000pattern} and recalled by
\cite{gratton2019direct}, if $\mathcal{C}$ is a set of cones $c_j$
that are respectively positively generated from a set of vectors
$G(c_j)$, then
$$\lambda(\mathcal{C}):= \min_{c_j \in \mathcal{C}}\left\{\inf_{u \in
    \rn, P_{c_j}(u)\neq 0} \max_{v \in G(c_j)} \frac{u^T
    v}{\|P_{c_j}(u)\| \|v\|}\right\}>0\;,$$ which guarantees that
$ cm_{T(x_k, \alpha_k)}(\mathcal{G}_k)$ is bounded by some constant
$\kappa_{min}>0$, under Assumption \ref{ass:gk}.  Under the above assumptions, we can bound the
regret of FDS-Plan 
when the optimum lies in the interior of the feasible set.
 Note that assuming optimal allocation in the interior of the feasible set is crucial for analysis, but we believe the opposite would not be harmful in practice, as supported by simulations (see Appendix \ref{app:suppl_exp}).

\begin{theorem}\label{th:FDS-Plan_hor}
   Under Assumptions \ref{ass:reg}, \ref{ass:bounded_D},
   and \ref{ass:gk},  if
   $x_\star \in \text{int}(\D)$ and if $|\mathbb{D}_k|$ is bounded by a constant $S_{\mathbb{D}}$, the cumulative regret $R_T$ of FDS-Plan
   (respectively FDS-Seq) after the first $T$ evaluations of $f$
   satisfies
$$\E[R_T]= O\big(\log(T)^{2/3}T^{2/3}\big)$$
for the choice $\delta = T^{-4/3}$ (respectively $\delta = T^{-10/3}$
for FDS-Seq).
\end{theorem}

In the absence of a lower bound, the optimality of such a regret rate
is unsure. It is difficult to compare it to other known bounds, as the
performance of related algorithms is often not evaluated in the same
way. In particular, the performance of the version of stochastic
gradient descent proposed by \cite{akhavan2020exploiting} is analyzed
with respect to a different notion of regret, $\tilde{R}_T$, that does
not take into account the samples needed for the estimation of each
gradient.  Their analysis yields $\tilde{R}_T = O(\sqrt{T})$. It is
important to note that the algorithm by \cite{akhavan2020exploiting}
takes advantage of the fact that in the setting of the latter paper,
sampling points outside of the feasible domain is possible.  When
using homothetic perturbation as we did for the illustration in Figure
\ref{fig:GD}, the regret of such a method is of the order of
$T^{2/3}$, as explained in Section \ref{sec:illustration}.\looseness-1

Black-box algorithms such as HOO or StoOO
\citep{bubeck2011x,munos2014bandits} are other possible baselines. When given a balanced hierarchical partition of
$\mathcal{X}$, and the smoothness of the function around its optimum,
these algorithms would incur a regret of the order of $\sqrt{T}$.  The
regret rate of FDS-Plan appears to be larger than that of HOO
instantiated with the right parameters. However, HOO relies on a
partition of the feasible domain that is computationally difficult to
achieve with arbitrary linearly constrained domains.


The assumption that $|\mathbb{D}_k|$ be bounded by a constant $S_{\mathbb{D}}$ is actually not constraining at all, since one way of satisfying Assumption \ref{ass:gk} is to set
$\mathbb{D}_k$ to the constant set of vectors corresponding to edges of optimization domains, which amounts to $2m(m-1)$ directions, where $m$ is the number of constraints.
Depending on the optimization domains, there may be smarter ways of choosing $|\mathbb{D}_k|$ that lead to smaller constants $S_{\mathbb{D}}$. The motivational case of resource allocation, where the feasible domain is the simplex, is an example of that.

In that case, the above method for choosing $\mathbb{D}_k$ yields $S_{\mathbb{D}}=2d(d+1)$ whereas recomputing the direction set at every round can spare us a factor $d$. An intuitive way to understand this is to consider the simplex of dimension $d=2$. When the iterate is in the interior of the simplex, and the step-size $\alpha_k$ is such that the admissible directions in $T(x_k, \alpha_k)$ form $\R^2$, we only need $d+1=3$ vectors (an angle of $2\pi/3$ apart). When $T(x_k, \alpha_k)$ is smaller, then minimal sets $G_k$ are formed by even fewer vectors. An efficient method for recomputing the set of directions at every round is described in \cite{griffin2008asynchronous}. It is possible to verify that for the simplex, this method provides less than $2d$ directions at each round. \looseness-1

This is particularly important, because the regret bound is
proportional to the number of directions contained in $\mathbb{D}_k$
(see the proof of Theorem \ref{th:FDS-Plan_hor}, in Appendix
\ref{app:interior_proof}), so that the dependence of the regret with
respect to $d$ is linear. For the sake of comparison,  the regret of the algorithm by \cite{akhavan2020exploiting} is quadratic in $d$, whereas, when HOO is perfectly parameterized, the dependence on $d$ of the regret of HOO disappears. Recall however, that on arbitrary linearly constrained domain, or even on the simplex in high dimension, HOO might be computationally intractable.

\subsubsection{Elements of Proof}
After some finite number of iterations that depends on
$\Delta:=\min_{i \in \{1 \ldots m\}}\mathrm{dist}(x_\star, \mathcal{C}_i)$,
the distance from $x_k$ to the boundaries of $\D$ is smaller than
$\Delta/4$ with high probability, thanks to the analysis of
\cite{gratton2019direct}. Waiting for another number of iterations,
$\alpha_k$ gets small enough for the approximate tangent cone
$T(x_k, \alpha_k)$ to describe the whole space $\R^d$. Then, the
trajectory of the algorithm is the same as in the unconstrained
setting.  In the unconstrained setting, the following elements provide
an intuition of why the regret is of the order of $T^{2/3}$.  With
similar arguments to those of the proof of Theorem
\ref{th:direct_search_reg_bound}, i.e Lemmas
\ref{lem:bound_increase_grad}, \ref{lem:gradnorm_failures} and
\ref{lem:sumsteps}, it is easy to see that the instantaneous regret
incurred at iteration $k$ of the algorithm is proportional to the sum
of $\alpha_k^{-2}$ (up to logarithmic factors), whereas it was
proportional to $\alpha_k^2$ in the deterministic case: indeed,
iteration $k$ now involves $N_k$ times more evaluations than in the
noiseless case and $N_k$ is proportional to $\alpha_k^{-4}$ (up to
logarithmic factors).  Thanks to Lemma \ref{lem:sumsteps}, we know
that $\alpha_k^2$ is summable.  Then, thanks to H\"{o}lder's
inequality applied to the sum of $\alpha_k^{-2}$ written as
$\left(\alpha_{k}\right)^{2/3}\left(\alpha_{k}\right)^{-2/3-2}$, the
regret is proportional to the total number of evaluations to the power
of $2/3$, up to logarithmic factors.  The complete proof can be found
in Appendix \ref{sec:app_noi_c}. \looseness-1

\section{Experiments}\label{sec:simu}

In our experiments, we focus on the case in which there are seven resources ($d=6$), and the loss functions are of the same form as in Section \ref{sec:illustration}, and 
$w_i(x)= -\tau_i \frac{\log(1+ \gamma x)} {\log(1+ \gamma )}$ with
$\gamma=2$, $\tau_1= 1$, $\tau_2= \tau_3 = \tau_4 = 0.75$, $\tau_5 = 0.89$, and $\tau_6 = \tau_7= 0.95$.
On Figure \ref{fig:expes}, we compare FDS-Seq and FDS-Plan to UCB on a discretization of the space, and gradient descent with an homothetic perturbation. Both methods are explained in Section \ref{sec:illustration}. The comparison with HOO is made impossible by the numerical complexity of HOO.  We set the horizon to $T = 500, 000$ and use a Gaussian noise with standard deviation $\sigma = 0.1$. The set of directions used in FDS-Seq and FDS-Plan are chosen with the method of \cite{griffin2008asynchronous}. The step parameter of the grid of UCB is set as $T^{-1/(d+2)}=T^{-1/8}$.

\begin{figure}[ht]
    \centering
    \includegraphics[width=0.6\linewidth]{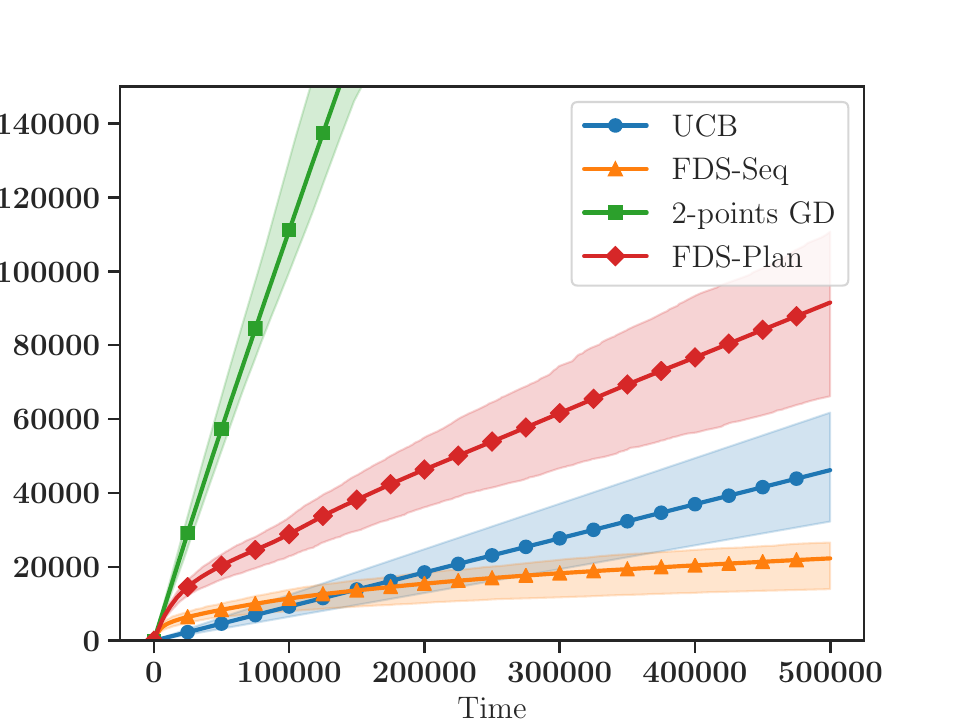}
    \caption{Regret plots of various strategies for resource allocation}
    \label{fig:expes}
\end{figure}

The regrets of the algorithms strongly depend on the chosen function. In the case
of UCB, the position of the maximizer with respect to the grid that
it relies on is important. To alleviate this issue, we plot the mean regret of
all the algorithms, when randomly shifting the loss function by a random vector
whose coordinates are in $[0.05, 0.05]$. We use 1200 Monte Carlo repetitions. The shaded area represents the region between the first and third quartile.

Clearly, FDS-Seq, by
reducing the number of samples needed at the beginning of the run (when moves
corresponds to significant drops of the target function), dominates FDS-Plan. 
The  unsatisfying 
performance of the gradient descent algorithm can be explained both by the homothetic 
perturbation that harms the convergence to the optimizer, and by the bad dependence of 
this algorithm on the dimension. FDS-Seq also clearly outperforms UCB. 
Note that the performance of UCB will worsen in
higher dimension due to the difficulty of simultaneously controlling the distance
between grid points and the overall number of points in the grid. \looseness-1

\section{Conclusion}
We have studied extensions of direct search algorithms designed for
linearly constrained zeroth-order optimization in the stochastic
setting. We have shown that these algorithms, though being fairly
simple, suffer a regret of the order of $T^{2/3}$, which is quite
satisfactory when compared to other options with comparable
implementation cost, like those inspired by
finitely-armed bandit algorithms or by gradient descent schemes. There
is still a performance gap, in terms of regret rate, when compared to some continuously-armed bandit
approaches that are however computationally much more heavy, even in low-dimensional
instances of the resource allocation model, such as the one considered in Section~\ref{sec:simu}.
We do not believe that the analysis of the algorithms proposed in this paper can be
refined so as to obtain the $\sqrt{T}$ regret rate. However, an interesting open question for
future work is to know whether this bound could be achieved by other sampling allocation
schemes.

\clearpage

\bibliography{ref_allocation}
\bibliographystyle{tmlr}

\clearpage
\appendix
\input{appendix}

\end{document}

%% file: math_commands.tex
\usepackage{amsmath,amsfonts,bm}

\def\eqref#1{equation~\ref{#1}}
\def\Eqref#1{Equation~\ref{#1}}

\def\1{\bm{1}}

\def\rn{{\mathbb{R}^d}}

\DeclareMathAlphabet{\mathsfit}{\encodingdefault}{\sfdefault}{m}{sl}
\SetMathAlphabet{\mathsfit}{bold}{\encodingdefault}{\sfdefault}{bx}{n}

\newcommand{\E}{\mathbb{E}}

\newcommand{\R}{\mathbb{R}}

\DeclareMathOperator*{\argmax}{arg\,max}
\DeclareMathOperator*{\argmin}{arg\,min}

%% file: appendix.tex
\appendix
{\Large\textbf{Supplementary Material}}\vspace*{1em}\\

\paragraph{Outline.} 
Appendix \ref{app:broader_impact}, contains general considerations about 
the potential repercutions
of this (methodological) work.
We prove in Appendix \ref{sec:app_det_unc} 
all the results pertaining to the noiseless, unconstrained case.
 In Appendix \ref{sec:noisy_unconstrained},
 we provide an additional result for the noisy but unconstrained  case. The analysis of direct search in the latter case paves the way for the proof of Theorem \ref{th:FDS-Plan_hor} 
 whose proof is deferred to Appendix \ref{sec:app_noi_c}.
Appendix \ref{app:HOO} contains further explanations about simulations in Section \ref{sec:illustration} and Appendix \ref{app:suppl_exp} contains additional experiments.

\section{Broader Impact Statement} \label{app:broader_impact}
This paper is mostly a methodological paper that is unlikely to have a direct societal impact.

However, it explores the idea that direct search algorithms, akin to approximate descent algorithms, can provide explicability in the context of budget allocation for advertising. 
Advertising practitioners who use these algorithms can explain their actions to their clients by guaranteeing that with
high probability, the latter result in an increase of the desired performance indicator. This work is thus part of a collective effort to reach explicability in machine learning, which is crucial as it
allows for more transparency.

From an even broader perspective, setting budgets for advertising campaigns is still a manual task in many companies, which could be replaced by algorithms such as those we propose here. Note that this would still leave the task of setting the scope of the campaign (which users to target, on which inventories, etc.) to marketing professionals. It is not clear which impact on employment the automation of budget allocation could have. 
However, for now, digital marketing is a flourishing sector where employment seems to have increased steadily in the last few years.

\section{Deterministic and unconstrained set-up}\label{sec:app_det_unc}
\subsection{Preliminary Results}

\begin{repeatlem}{lem:bound_increase_grad} If $f$ satisfies Assumption \ref{ass:reg}, 
$$\forall k'>k, \|\nabla f(x_{k'}) \|\leq \frac{\beta}{a}\|\nabla f(x_{k}) \|$$
\end{repeatlem}

\begin{proof}
First observe that because of strong convexity,
$$(\nabla f(x_k) - \nabla f({x_\star}))^{\top}(x_k-{x_\star})\geq a \|x_k-{x_\star}\|^2, $$ and $$a \| x_k- x^*\|^2\leq \|\nabla f(x_k)\| \| x_k- x^*\| ,$$
which implies that
\begin{equation}\label{eq:strong_cvx}
 a\| x_k- x^*\|\leq \|\nabla f(x_k)\|\;.\end{equation}
 Hence, 
\begin{align*}
\|\nabla f(x_{k'})\| &\leq \beta \|x_{k'} - {x_\star}\|\\
& \leq \frac{\beta}{\sqrt{a}} \sqrt{f(x_{k'} )- f({x_\star})} \leq \frac{\beta}{\sqrt{a}} \sqrt{f(x_{k} )- f({x_\star})} 
\\
& \leq \frac{\beta}{\sqrt{a}} \sqrt{\nabla f(x_k)
^\top (x_k - x^*)}
\leq \frac{\beta}{\sqrt{a}} \sqrt{\|\nabla f(x_k)\| \|
(x_k - x^*)\|}\\
& \leq \frac{\beta}{a} \nabla \|f(x_{k})\|\;,
\end{align*}
where the first inequality comes from the smoothness (in fact $\| \nabla f(x_{k})\|= \| \nabla f(x_{k})\|- \| \nabla f(x^*)\|  \leq \beta \|x_k - {x_\star}\|$), the second inequality is a result of  the strong convexity,  the third one ensues from  the fact that the algorithm is a descent algorithm, the fourth one arises as a resut of  convexity and the fifth one comes from the strong convexity property of \Eqref{eq:strong_cvx}.

\end{proof}

\begin{repeatlem}{lem:gradnorm_failures}
If $f$ satisfies Assumption \ref{ass:reg} and  iteration $k$ corresponds to an unsuccessful iteration then 
$$\|\nabla f(x_k) \|\leq \frac{1}{\kappa}\left(\frac{\beta}{2} \alpha_k + \frac{\rho(\alpha_k)}{\alpha_k}\right)= \frac{1}{\kappa}\left(\frac{\beta}{2}  + c \right)\alpha_k\;.$$
\end{repeatlem}
This lemma is already well-known \citep[see e.g. ][]{vicente2013worst}, we only prove it here for completeness.
\begin{proof}
Since $cm(\mathbb{D}_k):= \min_{u \in \rn, u\neq 0} \max_{v \in \mathbb{D}_k} \frac{u^T v}{\|u\| \|v\|} > \kappa$,
there exists $v \in \mathbb{D}_k$ such that 
$$-\nabla f(x_k)^{\top}v\geq \kappa\|\nabla f(x_k)\|\; .$$
Since the iteration is an unsuccessful iteration, we have 
$f(x_k)- f(x_k+\alpha_k v) \leq  \rho(\alpha_k) = c \alpha_k^2$.
Then 
\begin{align*}
 \kappa \alpha_k \|\nabla f(x_k)\|  - \rho(\alpha_k) &\leq - \nabla f(x_k)^{\top}v + f(x_k+\alpha_k v)- f(x_k) \\
 &\leq \int_{0}^{\alpha_k} \nabla f(x_k + u v)^{\top}v - \nabla f(x_k)^{\top}v du\\
 &\leq \int_{0}^{\alpha_k} \|\nabla f(x_k + u v) - \nabla f(x_k)\| \|v\| du\\
 &\leq \beta \int_0^{\alpha_k} u du \leq \frac{\beta}{2} \alpha_k^2 \;,
 \end{align*}
 which yields
 $\|\nabla f(x_k) \|\leq \frac{1}{\kappa}\left(\frac{\beta}{2} \alpha_k + \frac{\rho(\alpha_k)}{\alpha_k}\right)= \frac{1}{\kappa}\left(\frac{\beta}{2} \alpha_k + c \alpha_k\right).$
\end{proof}

\begin{repeatlem}{lem:sumsteps}
$$\sum_{k=0}^{\infty}\rho(\alpha_k)\leq \frac{1}{1- \theta^2}(f(x_{1}) - f({x_\star}) +\rho(\alpha_0))\;. $$
\end{repeatlem}
This lemma is  also a common element of the analysis of direct search algorithms \citep[see e.g. ][]{gratton2019direct}, we only prove it here for completeness.

We assume that there are infinitely many successful iterations as it is trivial to adapt the argument otherwise.
Let $k_i$ be the index of the $i$th successful iteration ($i \geq 1$).  Define $k_0= -1$
and $\alpha_{-1}=  \alpha_0 $ for convenience.  Let us rewrite
$\sum_{k=0}^{\infty}\rho(\alpha_k)$ as $\sum_{i=0}^{\infty}\sum_{k = k_i+1}^{k_{i+1}}\rho(\alpha_k)$ and study first $\sum_{k = k_i+1}^{k_{i+1}}\rho(\alpha_k).$
Thanks to the definition of the update on a successful iteration and on unsuccessful iterations, 
  \begin{align*}\sum_{k = k_i+1}^{k_{i+1}}\rho(\alpha_k)=
  \sum_{k = k_i+1}^{k_{i+1}}\rho(\theta^i  \alpha_{k_i}) =
  \sum_{k = k_i+1}^{k_{i+1}}c\theta^{2 i} \rho(\alpha_{k_i})\leq \frac{1}{1- \theta^2} \rho(\alpha_{k_i})\;.
  \end{align*}
  Since on successes $\rho(\alpha_{k_i})\leq f(x_{k_i}) -f(x_{k_i}+ \alpha_{k_i}) = f(x_{k_i}) -f(x_{k_{i+1}}) $,
$$\sum_{i=1}^{\infty}\rho(\alpha_{k_i})\leq f(x_{1}) - f({x_\star})\;.$$
Hence $$\sum_{k=0}^{\infty}\rho(\alpha_k)\leq \frac{1}{1- \theta^2}(f(x_{1}) - f({x_\star}) +\rho(\alpha_0))\;. $$

\begin{lemma}\label{lem:k_f}The index $k_f$  of the first unsuccessful iteration satisfies:
$$k_f\leq  \frac{f(x_{0})-f(x_{*})}{\rho(\alpha_0)} \;.$$
\end{lemma}
Another version of this lemma is due to \cite{gratton2015direct}.
\begin{proof}
Before the first unsuccessful iteration, $\alpha_k = \alpha_0$. So by definition of a successful iteration, $\forall k, \text{ such that } 0<k\leq k_f$
$$f(x_{k-1}) - f(x_{k}) \geq \rho(\alpha_0)\;.$$

\noindent By summing, 
$$f(x_{k_0}) - f(x_{k_f}) \geq k_f \rho(\alpha_0).$$
The left hand-side of this inequality is upper-bounded by $f(x_{0})-f(x_{*})$, which suffices to conclude the proof.
\end{proof}
\subsection{Regret Bound}
In this section, we prove in Theorem \ref{th:direct_search_reg_bound_app} below a result involving the regret at iteration $K$ of the algorithm  instead of the regret after $T$ function evaluations.
As $T\leq K(S_{\mathbb{D}}+1)$, Theorem \ref{th:direct_search_reg_bound_app} directly implies Theorem \ref{th:direct_search_reg_bound}.

Consider
  $$\tilde{R}_K =\sum_{k=0}^K \left(f(x_k) - f({x_\star}) + \sum_{v\in \mathbb{D}_k} f(x_k+ \alpha_k v) - f({x_\star})\right)$$ which is an upper bound of the cumulative regret suffered by the algorithm at iteration $K$, since it accounts for all directions in $\mathbb{D}_k$ at each round $k$, while not necessarily all of them will be tested. $\tilde{R}_K$ can be bounded as follows.

\begin{theorem}\label{th:direct_search_reg_bound_app}
Under Assumptions \ref{ass:reg} and \ref{ass:kappa},
\begin{align*}\tilde{R}_K &\leq (S_{\mathbb{D}}+1)\Big[\left(\frac{1}{c}\left(\frac{1}{1- \theta^2}(f(x_{0}) - f({x_\star}) +\rho(\alpha_0)\right)\right)\left(\left(1+ \frac{\eta}{a}\right) \eta + \beta \right) \\
&+\frac{f(x_{0}) - f({x_\star})}{\rho(\alpha_0)} \left(\frac{\beta}{a}\|\nabla f(x_0) \| \alpha_0 + \beta \alpha_0 ^2\right)\Big]\;,
\end{align*}
where $\eta :=\frac{\beta}{a}\frac{1}{\kappa \theta} (c + \frac{\beta}{2})   $.
\end{theorem}

\begin{proof}
We decompose the regret as  
\begin{align}
\tilde{R}_K &= \sum_{k=0}^K \left(f(x_k) - f({x_\star}) + \sum_{v\in \mathbb{D}} f(x_k+ \alpha_k v) - f({x_\star})\right)\nonumber\\
&\leq  \sum_{k=0}^K \left(f(x_k) - f({x_\star}) + \sum_{v\in \mathbb{D}} f(x_k+ \alpha_k v) - f(x_k)+ f(x_k)- f({x_\star})\right)\nonumber\\
&\leq \sum_{k=0}^K \left(
(\mathbb{D}|+1)(f(x_k) - f({x_\star}))\right) + \sum_{k=0}^K\left(\sum_{v\in \mathbb{D}}  f(x_k+ \alpha_k v) - f(x_k)
\right)\nonumber\\
&\leq \sum_{k=0}^{k_f} \left((\mathbb{D}|+1)(f(x_k) - f({x_\star}))\right) + \sum_{k=0}^{k_f}\left(\sum_{v\in \mathbb{D}}  f(x_k+ \alpha_k v) - f(x_k)\right) \nonumber\\
& + \sum_{k=k_f}^K \left((\mathbb{D}|+1)(f(x_k) - f({x_\star}))\right) + \sum_{k=k_f}^K\left(\sum_{v\in \mathbb{D}}  f(x_k+ \alpha_k v) - f(x_k)\right)\label{eq:decomp_reg},
\end{align}
where $k_f $ is the iteration of the first unsuccessful iteration. The third inequality provides a decomposition of the regret in a first term that involves the suboptimality of the iterate, and a second term that involves the difference between values of $f$ at the iterate and at the trial points. 
As is usual for direct search algorithm, the behavior of the algorithm before the first unsuccessful iteration has to be studied separately, which explains the use of the decomposition of the fourth inequality.
We bound the regret due to the rounds preceding $k_f$ by:
\begin{lemma} \label{lem:reg_k_g}
The regret due to the rounds preceding $k_f$ is bounded by
$$\sum_{k=0}^{k_f} \left((\mathbb{D}|+1)(f(x_k) - f({x_\star}))\right) + \sum_{k=0}^{k_f}\left(\sum_{v\in \mathbb{D}}  f(x_k+ \alpha_k v) - f(x_k)\right)\leq C_1,$$
where we denote by $k_f$ is the index of the first unsuccessful iteration and by \\
$C_1 =\frac{f(x_0) - f({x_\star})}{c \alpha_0^2} \left((f(x_0) - f({x_\star})+ \frac{\beta}{a}\|\nabla f(x_0)\| \alpha_0 + \beta \alpha_0 ^2\right)$.
\end{lemma}
\begin{proof}
As until $k_f$,  $f(x_{k}) \leq f(x_0)$ and $\alpha_k = \alpha_0$, it holds that
\begin{align*}
&\sum_{k=0}^{k_f} \left((\mathbb{D}|+1)(f(x_k) - f({x_\star}))\right)+ \sum_{k=0}^{k_f}\left(\sum_{v\in \mathbb{D}}  f(x_k+ \alpha_k v) - f(x_k)\right)\\
&\leq \sum_{k=0}^{k_f} \left((\mathbb{D}|+1)(f(x_0) - f({x_\star}))\right)+ \sum_{k=0}^{k_f}\left(\sum_{v\in \mathbb{D}}  \|\nabla f(x_k)\| \alpha_k + \beta \alpha_k ^2\right)\\
&\leq k_f \left((\mathbb{D}|+1)(f(x_0) - f({x_\star}))\right) + (\mathbb{D}|+1)k_f\left(  \frac{\beta}{a}\|\nabla f(x_0)\| \alpha_0 + \beta \alpha_0 ^2\right)\\
&\leq \frac{f(x_0) - f({x_\star})}{c \alpha_0^2} \left((\mathbb{D}|+1)(f(x_0) - f({x_\star})+ \frac{\beta}{a}\|\nabla f(x_0)\| \alpha_0 + \beta \alpha_0 ^2)\right) \\
&= C_1\;,
\end{align*}
where Lemma \ref{lem:bound_increase_grad} is used for the second inequality and the third inequality comes from Lemma \ref{lem:k_f}. The first inequality results from the following property of convex and $\beta$-smooth functions:
$f(y)-f(x) \leq \nabla{f}(y)^T(x-y)\leq \| \nabla{f}(x)^T(x-y) \|+\beta \|x-y\|^2 $, applied to $x_k+ \alpha_k v$ and $x_k$.
\end{proof}

\begin{lemma}\label{lem:reg_2part} After $k_f$, 
$$\sum_{k=k_f}^K   f(x_k+ \alpha_k v) - f(x_k)\leq  C_2$$
where $k_f$ is the index of the first unsuccessful iteration and \\
 $C_2=\frac 1 c \left(
\eta + \beta
\right) \left(
\frac{1}{1- \theta^2}(f(x_{0}) - f({x_\star}) +\rho(\alpha_0)\right)\;.$

\end{lemma}
\begin{proof}
We take $k>k_f.$ Using  the property of convex and $\beta$-smooth functions that 
$f(y)-f(x) \leq \| \nabla{f}(x)^T(x-y) \|+\beta \|x-y\|^2 $, applied to $x_k+ \alpha_k v$ and $x_k$, as in Lemma \ref{lem:reg_k_g}, we get
$$
f(x_k+ \alpha_k v) - f(x_k)\leq  \alpha_k  \|\nabla f (x_k)\| + \beta \alpha_k^2\;.
$$
We note that if $k$ is the index of an unsuccessful iteration,
$$\|\nabla f (x_k)\|\leq \frac{1}{\kappa} \left(c + \frac{\beta}{2}\right)\alpha_k = \frac{1}{L_1'}\alpha_k\;,$$
by Lemma \ref{lem:gradnorm_failures}, if $\alpha_k \leq 1$.
If $k$ is the index of a successful iteration, we can come back to the last unsuccessful iteration $k'$, since 
\begin{align*}
\|\nabla f (x_k)\|&\leq \frac{\beta}{a}\|\nabla f (x_{k'})\|\leq\frac{\beta}{a} \frac{1}{\kappa} \left(c + \frac{\beta}{2}\right)\alpha_{k'}\leq\frac{\beta}{a} \frac{1}{\kappa} \left(c + \frac{\beta}{2}\right)\frac{\alpha_{k}}{\theta}\;.
\end{align*}
where the first inequality comes from Lemma \ref{lem:bound_increase_grad}, and the third from the fact that $\alpha_k \geq \theta \alpha_{k'}$.
Hence for any $k>k_f$,
\begin{align*}
\|\nabla f (x_k)\|\leq \eta \alpha_k\;, 
\end{align*}
and 
\begin{align*}
f(x_k+ \alpha_k v) - f(x_k)\leq (\eta+ \beta) \alpha_k\;.
\end{align*}

Hence $$\sum_{k=0}^K   f(x_k+ \alpha_k v) - f(x_k)\leq  \sum_{k=0}^K                                                                                                                                                                                                                                                                                                                                                                                                                                                                                                                                                                                                                                                                                                                                                                                                                                                                                                                                                                                                                                                                                                                                                                                                                                                                                                                                                                                                                                                                                                                                                                                                                                                                                                                                                                                                                                                                                                                                                                                                                                                                                                                                                     \left(\eta+ \beta\right)\alpha_{k}^{2}\;. $$

Consequently, 
\begin{align*}\sum_{k=0}^K   f(x_k+ \alpha_k v) - f(x_k)
&\leq   \left(
\eta + \beta
\right) 
\frac{1}{c} \left(\sum_{k=0}^{\infty} \rho(\alpha_k)
\right)  \\
&\leq  \frac 1 c \left(
\eta + \beta
\right) \left(
\frac{1}{1- \theta^2}(f(x_{0}) - f({x_\star}) +\rho(\alpha_0))
\right)\;. 
\end{align*}
\end{proof}

\begin{lemma}\label{lem:reg_3part}  
\begin{align*} 
\sum_{k=k_f}^K (f(x_k) - f({x_\star})) \leq \frac{1}{ac}  \eta^2 \left(\frac{1}{1- \theta^2}(f(x_{1}) - f({x_\star}) +\rho(\alpha_0))\right) := C_3  
\end{align*}
\end{lemma}

\begin{proof}
Take $k>k_f$. Thanks to the convexity of $f$,
\begin{align*}
f(x_k) - f({x_\star}) &\leq \nabla f(x_k)^\top (x_k-{x_\star})\\
&\leq \frac{1}{a}\|\nabla f(x_k)\|^2\;,
\end{align*}
where the second inequality stems from \Eqref{eq:strong_cvx}, which itself come from strong convexity.

As in the proof of Lemma \ref{lem:reg_2part}  we have for any $k>k_f$,
\begin{align*}
\|\nabla f (x_k)\|\leq \eta \alpha_k\;,
\end{align*}
so that for any $k>k_f$,
\begin{align*}
f(x_k) - f({x_\star}) &\leq \frac{1}{a}  \left(\eta\right)^2\left(\frac{\alpha_{k}}{\theta}\right)^{2}.
\end{align*}
Thanks to Lemma \ref{lem:sumsteps}, we have $\sum_{k=0}^{\infty}\alpha_k^2\leq \left(\frac{1}{c}\frac{1}{1- \theta^2}(f(x_{1}) - f({x_\star}) +\rho(\alpha_0)\right).$

Eventually, 

$$f(x_k) - f({x_\star}) \leq 
\frac{1}{a}  \eta^2 \left(
\frac{1}{c}\frac{1}{1- \theta^2}(f(x_{1}) - f({x_\star}) +\rho(\alpha_0))
\right)=C_3\;.$$
\end{proof}
Using  the regret decomposition of  \Eqref{eq:decomp_reg} together with Lemmas \ref{lem:k_f}, \ref{lem:reg_2part}, and \ref{lem:reg_3part} completes the proof of Theorem \ref{th:direct_search_reg_bound_app}.
\end{proof}

\section{Noisy and unconstrained set-up}\label{sec:noisy_unconstrained}

Before considering the constrained setting, we  analyze the algorithms described in Section \ref{sec:algorithms} (Algorithms \ref{alg:con_dir_search_f_samples} and \ref{alg:f_dir_search_f_tests}) when there are no constraints, that is, $\D = \mathbb{R}^d$. 
\subsection{Presentation of the main Result}

\begin{theorem}\label{th:DSwS_hor}
 Assume that $f$ is lower bounded and upper bounded on $\R^d$, so that there exits $U$, $f(x) -f({x_\star})\leq U, ~ \forall x \in \R^d$. Also assume that the region {$\mathcal{X}= \{x\in \R^d :\, f(x)<f(x_0)\}$} is convex and that $f$ is $a$-strongly convex and $\beta$-smooth  on $\mathcal{X}$. 
Let $R_T$ be the cumulative regret on the $T$ first evaluations of $f$ made by FDS-Plan. Set $\delta = T^{-4/3}$. Then
$$\E[{R_T}]= O(\log(T)^{2/3}T^{2/3})$$
\end{theorem}

This regret bound is also valid for FDS-Seq under the same Assumptions, with $\delta = T^{-10/7}/2$. In the following, we give a proof of the regret bound for FDS-Plan.
Note that Sections \ref{sec:app_reg_wS_n_u} and \ref{sec:app_interm_u_n} refer to FDS-Plan, and Section \ref{sec:app_reg_wT_n_u} deals with FDS-Seq. \\
The regularity assumption in  Theorem \ref{th:DSwS_hor} requires that $f$  is bounded and  satisfies a local version of Assumption \ref{ass:reg}. The initial point $x_0$ should not be chosen too far from $x^*$, nor should $\alpha_0$ be too large. This assumption is not unreasonable, since for every $x_0$ and $\alpha_0$, it is naturally satisfied by bounded and strictly convex functions in $\mathscr{C}^2$ for some choice of $a \text{ and } \beta$. We stress that under the alternative assumption \ref{ass:reg}, the same kind of regret bound could still be proved, but with a smaller choice of $\delta$, resulting in higher confidence bonuses and the multiplication of the regret by some constant factor.
Indeed, in this case, estimating $f$ incorrectly  at each round can lead to a trajectory  that always deviates from $x^*$, which is highly detrimental to the regret rate; meanwhile, under the assumption required by Theorem \ref{th:DSwS_hor}, $f$ is bounded by $U$, so that deviating from $x^*$ contributes to the regret by at most $UT$.

In the following, we will use the following  additional notation.
\paragraph{Notation.}
We define $v_k$ to be 
$$v_k :=\begin{cases} \argmax_{v\in \mathbb{D}_k} f(x_k)- f(x_k - \alpha_k v) \text{ if  iteration } k \text{ is unsuccessful}\\
\text{the chosen direction otherwise. }
\end{cases}$$

\subsection{Intermediate results}\label{sec:app_interm_u_n}
\begin{lemma}\label{lem:E_k}
We call $\mathcal{E}_k$ the event $$\mathcal{E}_k =\{ |f(x_k+ \alpha_k v)- \hat{f}(x_k+ \alpha_k v)|\leq c/4(\alpha_k)^2\},~\forall v \in {\mathbb D_k } \cup\{0\}\}\;.$$
The probability of $\mathcal{E}_k$ is lower bounded by
$$\pr{\mathcal{E}_k|\mathcal{F}_{k-1}} \geq 1- \delta  (S_{\mathbb{D}}+1)$$ where $\mathcal{F}_{k-1}$ is the $\sigma$-field representing the history.
\end{lemma}
\begin{proof} Let $v \in {\mathbb D } \cup\{0\}$
$f(x_k + \alpha_k v )- \hat{f}(x_k {+} \alpha_k v )= \sum_{i=1}^{N_k}\epsilon_j$ with $\epsilon_j$ independent Gaussian variables with variance $\sigma^2$ and we have
$$\left| \sum_{i=1}^{N_j}\epsilon_j \right| \leq \sqrt{2 \sigma^2\frac{\log(2/\delta)}{N_j}}\leq \sqrt{\frac{2 \sigma^2\log(2/\delta)}{32 \sigma^2 \log(2/\delta)/\rho(\alpha_k)^2}}\leq \frac{\rho(\alpha_k)}{4}$$
with probability $1-\delta$, when knowing $N_k$.
By a union bound, $ \pr{\mathcal{E}_k|\mathcal{F}_{k-1}} \geq 1- \delta  (|\mathbb D|+1)$ where $\mathcal{F}_{k-1}$ is the $\sigma$-field representing the history.
\end{proof}
The following lemma characterizes unsuccessful iterations and successes when $\mathcal{E}_k$ occurs.
\begin{lemma}\label{lem:concentration}
On $\mathcal{E}_k$, if $k$ is an unsuccessful iteration then
$f(x_k)- f(x_k+\alpha_k v_k) \leq 3 c/2(\alpha_k)^2$ and
 if $k$ is a successful iteration then
$f(x_k)- f(x_k+\alpha_k v_k) \geq c/2(\alpha_k)^2$.
\end{lemma}
Lemma \ref{lem:concentration}  implies that if $\mathcal{E}_k$ occurs for all $k$, then each iteration of the algorithm results in a descent.
\begin{lemma}\label{lem:desc}
On $\cap_{k\leq K}\mathcal{E}_k$, the algorithm is a descent  algorithm. In particular, $x_k \in \mathcal{X}, ~\forall k \in  \{1 \ldots K\}$.
\end{lemma}

\begin{lemma}\label{lem:bound_increase_grad_noise}If $f$ satisfies the assumptions of Theorem \ref{th:DSwS_hor} and on $\cap_{k\leq K}\mathcal{E}_k$ then, 
$$\forall k'>k, \|\nabla f(x_{k'}) \|\leq \frac{\beta}{a}\|\nabla f(x_{k}) \|\;.$$
\end{lemma}
\begin{proof}
The proof of Lemma \ref{lem:bound_increase_grad} applies verbatim thanks to Lemma \ref{lem:desc}.
\end{proof}
\begin{lemma}\label{lem:gradnorm_unsuccessful iterations_noise}
If $f$ satisfies the assumptions of Theorem \ref{th:DSwS_hor} and the iteration $k$ corresponds to an unsuccessful iteration then on $\mathcal{E}_k$,
$$\|\nabla f(x_k) \|\leq \frac{1}{\kappa}\left(\frac{\beta}{2} \alpha_k + \frac{3\rho(\alpha_k)}{2\alpha_k}\right)= \frac{1}{2\kappa}\left(\beta \alpha_k + 3c \alpha_k^2\right)\;.$$

\end{lemma}
\begin{proof}
We reproduce the proof of Lemma \ref{lem:gradnorm_failures} by using Lemma  \ref{lem:concentration}.

Since $cm(\mathbb{D}):= \min_{v \in \rn} \max_{v \in \mathbb{D}}{\frac{v^Tv}{\|v\|\|v\|}}> \kappa$,
there exists $v \in \mathbb{D}$ such that 
$$-f(x_k)^{\top}v\geq \kappa\|\nabla f(x_k)\| .$$
Since the iteration is an unsuccessful iteration, we have 
$f(x_k)- f(x_k+\alpha_k v) \leq  \frac{3}{2}\rho(\alpha_k) = \frac{3}{2}c \alpha_k^2$, thanks to Lemma \ref{lem:concentration}. 
Then 
\begin{align*}
 \kappa \alpha_k \|\nabla f(x_k)\|  - \rho(\alpha_k) \leq \frac{\beta}{2} \alpha_k^2 \;,
 \end{align*} exactly as in the proof of Lemma \ref{lem:gradnorm_failures},
 which yields
 $\|\nabla f(x_k) \|\leq \frac{1}{\kappa}\left(\frac{\beta}{2} \alpha_k + \frac{3}{2}\frac{\rho(\alpha_k)}{\alpha_k}\right)= \frac{1}{\kappa}\left(\frac{\beta}{2} \alpha_k +\frac{3}{2} c \alpha_k\right).$
\end{proof}

\begin{lemma}\label{lem:sumsteps_noise} If $f$ satisfies the assumptions of Theorem \ref{th:DSwS_hor} and on $\cap_{k\leq K}\mathcal{E}_k$,
$$\sum_{k=0}^{K}\rho(\alpha_k)\leq \frac{2}{1- \theta^2}(f(x_{1}) - f({x_\star}) +\rho(\alpha_0))\;. $$
\end{lemma}

Assume that $\cap_{k\leq K}\mathcal{E}_k$ holds.
Let $k_i$ be the index of the $i$-th successful iteration ($i \geq 1$).  Define $k_0= -1$
and $\alpha_{-1}=  \alpha_0 $, and $\alpha_k=0, ~ \forall k>K$ for convenience.  Define $K_I$ the number of successes until $K$.  We rewrite
$\sum_{k=0}^{K_I}\rho(\alpha_k)$ as $\sum_{i=0}^{K_I}\sum_{k = k_i+1}^{k_{i+1}}\rho(\alpha_k)$ and study first $\sum_{k = k_i+1}^{k_{i+1}}\rho(\alpha_k).$

  Thanks to the definition of the update on a successful iteration and on unsuccessful iterations, 
  \begin{align*}\sum_{k = k_i+1}^{k_{i+1}}\rho(\alpha_k)
  &\leq \frac{1}{1- \theta^2} \rho(\alpha_{k_i})\;.
  \end{align*}
  exactly as in the proof of Lemma \ref{lem:sumsteps}.
  Since on successes,  $$\frac{1}{2}\rho(\alpha_{k_i})\leq f(x_{k_i}) -f(x_{k_i}+ \alpha_{k_i}) = f(x_{k_i}) -f(x_{k_{i+1}})\;, $$ we have
$$\frac{1}{2}\sum_{i=1}^{K_I}\rho(\alpha_{k_i})\leq f(x_{1}) - f({x_\star}).$$
Hence $$\sum_{k=0}^{K_I}\rho(\alpha_k)\leq \frac{2}{1- \theta^2}(f(x_{1}) - f({x_\star}) +\rho(\alpha_0))\;. $$

\begin{lemma}\label{lem:k_f_noise}On $\cap_{k\leq K}\mathcal{E}_k$, the first iteration that results in an unsuccessful iteration occurs at round $k_f$, satisfying:
$$k_f\leq 2 \frac{f(x_{0})-f(x_{*})}{\rho(\alpha_0)} \;.$$
\end{lemma}
Before the first unsuccessful iteration, $\alpha_k = \alpha_0$. So by Lemma \ref{lem:concentration}, $\forall 0<k\leq k_f$
$$f(x_{k-1}) - f(x_{k}) \geq\frac{1}{2} \rho(\alpha_0)\;.$$

By summing, 
$$f(x_{k_0}) - f(x_{k_f}) \geq \frac{1}{2} k_f \rho(\alpha_0)\;.$$
The left hand-side of this inequality is upper-bounded by $f(x_{0})-f(x_{*})$, which suffices to conclude the proof.

\begin{lemma}\label{lem:gra_al_noise}If $f$ satisfies the assumptions of Theorem \ref{th:DSwS_hor} and on $\cap_{k\leq K}\mathcal{E}_k$, for any $k$ after the first unsuccessful iteration, 
$$\|\nabla f (x_k)\|\leq \eta_2 \alpha_k,$$
where we denote by  $\eta_2 = \frac{\beta}{2a\kappa\theta}(3c + \beta)$.
\end{lemma}
\begin{proof}
For unsuccessful iterations, 
$$\|\nabla f(x_k) \|\leq \frac{1}{2\kappa} (3c + \beta)\alpha_{k}\;,$$
 thanks to Lemma \ref{lem:gradnorm_unsuccessful iterations_noise}.
If $k$ is the index of a successful iteration, we can come back to the last unsuccessful iteration $k'$, since 
\begin{align*}
\|\nabla f (x_k)\|&\leq \frac{\beta}{a}\|\nabla f (x_{k'})\|\leq \frac{\beta}{a}\frac{1}{2\kappa} (3c + \beta)\alpha_{k'}\leq \frac{\beta}{a}\frac{1}{2\kappa} (3c + \beta)\left(\frac{\alpha_{k}}{\theta}\right)\;,
\end{align*}
where the first inequality comes from Lemma \ref{lem:bound_increase_grad_noise}, the second from Lemma  \ref{lem:gradnorm_unsuccessful iterations_noise} and the third from the fact that $\alpha_k \geq \theta \alpha_{k'}$.
\end{proof}
\subsection{Regret Analysis of FDS-Plan}\label{sec:app_reg_wS_n_u}

\begin{lemma}\label{lem:DSwS_hor}
If $f$ satisfies the assumptions of Theorem \ref{th:DSwS_hor} and
on $\cap_{k\leq K}\mathcal{E}_k$, $$\tilde{R}_K \leq C_4 log (2/\delta)  +  C_5 \log(2/\delta)\left(\sum_{k=1}^K N_k\right)^{2/3}\;,$$
where $\begin{cases} C_4 = \frac{32}{c^2} C_1 \alpha_0^{-4}(S_{\mathbb{D}}+1)\\
C_5= \frac{32}{c^2}(S_{\mathbb{D}}+1) \left(\frac{c}{32} \frac{1}{(1- \theta^2)}(f(x_{1}) - f({x_\star}) +\rho(\alpha_0))\right)^{1/3} 
\left(\frac{1}{a}\eta_2^2 +
		\eta_2 + \beta\right)\;. \end{cases}
$
\end{lemma}
\begin{proof}
In the following we study the case where $\cap_{k\leq K} \mathcal{E}_k$ holds true.

As in the deterministic case, we decompose the regret as 
\begin{align*}
\tilde{R}_K &=   \sum_{k=0}^{k_f}N_k \left(f(x_k) - f({x_\star}) + \sum_{v\in \mathbb{D}} f(x_k+ \alpha_k v) - f({x_\star})\right)\\
&+ \sum_{k=k_f}^K N_k\left(f(x_k) - f({x_\star}) + \sum_{v\in \mathbb{D}} f(x_k+ \alpha_k v) - f({x_\star})\right).
\end{align*}
We start by dealing with the cumulative regret before $k_f$.
We write \begin{align*}
&\sum_{k=0}^{k_f} N_k\left(f(x_k) - f({x_\star}) + \sum_{v\in \mathbb{D}} f(x_k+ \alpha_k v) - f({x_\star})\right)\\
& \leq N_0 \sum_{k=0}^{k_f}\left(f(x_k) - f({x_\star}) + \sum_{v\in \mathbb{D}} f(x_k+ \alpha_k v) - f({x_\star})\right)\\
& \leq \frac{32}{c^2}  \alpha_0^{-4} \log (2/\delta) \sum_{k=0}^{k_f}\left(f(x_k) - f({x_\star}) + \sum_{v\in \mathbb{D}} f(x_k+ \alpha_k v) - f({x_\star})\right)\\
& \leq \frac{32}{c^2}  \alpha_0^{-4} \log (2/\delta) \times 2  C_1\\
&= C_4 \log (2/\delta),
\end{align*}
where the last inequality is obtained exactly as in the proof of Lemma \ref{lem:reg_k_g} with the help of Lemma \ref{lem:k_f_noise} instead of Lemma \ref{lem:k_f}.
By using the above inequality and  the decomposition of the regret, we get
\begin{align*}
&\tilde{R}_K- C_4 \log (2/\delta)\\  &\leq  \sum_{k=k_f}^K N_k\left(f(x_k) - f({x_\star}) + \sum_{v\in \mathbb{D}} f(x_k+ \alpha_k v) - f({x_\star})\right)\\
&\leq \sum_{k=k_f}^K N_k \left(f(x_k) - f({x_\star}) + \sum_{v\in \mathbb{D}} f(X_k+ \alpha_k v) - f(x_k)+ f(x_k)- f({x_\star})\right)\\
& \leq \sum_{k=k_f}^K N_k \left((S_{\mathbb{D}}+1)(f(x_k) - f({x_\star})) + \sum_{v\in \mathbb{D}}  f(x_k+ \alpha_k v) - f(x_k)\right)\\
&\leq (S_{\mathbb{D}}+1) \sum_{k=k_f}^K N_k \left(\frac{1}{a}\|\nabla f(x_k )\|^2 + \| \nabla f(x_k )\| \alpha_k + \beta \alpha_k^2\right),
\end{align*}
where $C_4 = \frac{32}{c^2} C_1 \alpha_0^{-4}(S_{\mathbb{D}}+1)$. 
The fourth inequality comes from the regularity assumptions required for Theorem \ref{th:FDS-Plan_hor} together with Lemma \ref{lem:desc}.
We use Lemma \ref{lem:gra_al_noise} to get that for any $k>k_f$,
$$\|\nabla f (x_k)\| \leq \eta_2 \alpha_k
.$$
Then 
$$\frac{1}{a}\|\nabla f(x_k )\|^2 + \|\nabla f(x_k )\|\alpha_k + \beta \alpha_k^2\leq
\frac{1}{a} \eta_2^2 \alpha_k^2+
		\eta_2 \alpha_k^2 + \beta \alpha_k^2\;.$$
We get 
\begin{align*}
 \sum_{k=k_f}^K N_k \left(\frac{1}{a}\|\nabla f(x_k )\|^2 + \| \nabla f(x_k )\| \alpha_k\right)&\leq C_6 \log(2/\delta) \sum_{k= k_f}^K \left(\alpha_{k}\right)^{-2}\\&
 \leq C_6 \log(2/\delta) \sum_{k= 0}^K \left(\alpha_{k}\right)^{-2}\; .
\end{align*}
where $C_6 = \left(\frac{1}{a}\eta_2^2 +
		\eta_2  + \beta\right) \frac{32}{c^2}$.
Consequently $$\tilde{R}_K \leq C_4 \log (2/\delta) + C_6 \log(2/\delta) \sum_{k= 0}^K \left(\alpha_{k}\right)^{-2} .$$
The number of function evaluations is defined as $$\sum_{k=0}^K N_k  = \frac{32 \log (2/\delta)}{c}\sum_{k=0}^K \alpha_k^{-4} \; .$$
Thanks to Lemma \ref{lem:sumsteps_noise}, 
$$\sum_{k=0}^{K}(\alpha_k)^2\leq \frac{2}{c(1- \theta^2)}(f(x_{1}) - f({x_\star}) +\rho(\alpha_0)) \;. $$
By H\"older's inequality, we get

\begin{align*}
\sum_{k= 0 }^K \left(\alpha_{k}\right)^{-2} &= \sum_{k= 0}^K \left(\alpha_{k}\right)^{2/3}\left(\alpha_{k}\right)^{-2/3-2} \\
&\leq \left(\sum_{k= 0}^K \left(\alpha_{k}\right)^{2/3\times 3}\right)^{1/3} \left(\sum_{k= 0}^K \left(\alpha_{k}\right)^{{-{8}}/3\times 3/2}\right)^{2/3}
\\
&\leq \left(\sum_{k= 0}^K \left(\alpha_{k}\right)^{2}\right)^{1/3} \left(\sum_{k= 0}^K \left(\alpha_{k}\right)^{-4}\right)^{2/3}\; .
\end{align*}
And thus \begin{align*}&\tilde{R}_K - C_4 \log (2/\delta)\\
&\leq  C_6\log(2/\delta)\left( \frac{2}{c(1- \theta^2)}(f(x_{1}) - f({x_\star}) +\rho(\alpha_0)\right)^{1/3} \left(\frac{c}{8\log(2/\delta)}{\sum_{k=1}^K N_k}\right)^{2/3}\\
&\leq C_5\log(2/\delta)\left(\frac{1}{\log(2/\delta)}{\sum_{k=1}^K N_k}\right)^{2/3}\\
& \leq C_5\log(2/\delta)^{1/3}\left({\sum_{k=1}^K N_k}\right)^{2/3}\;,
\end{align*}
where $C_5= (S_{\mathbb{D}}+1) C_6 \left(\frac{c}{32} \frac{1}{(1- \theta^2)}(f(x_{1}) - f({x_\star}) +\rho(\alpha_0))\right)^{1/3}$.

\end{proof}

\begin{proof} \textbf{of Theorem \ref{th:DSwS_hor}}
We note $K_T$ the last round reached by the algorithm with $T$ evaluations. 
Lemma \ref{lem:DSwS_hor} proves that on $\cap_{k\leq K_T} \mathcal{E}_k$, 
$$\tilde{R}_{K_T} \leq C_4  \log (2/\delta) +  C_5 \left(\frac{1}{(S_{\mathbb{D}}+1)}\right)^{2/3} \log(2/ \delta)^{2/3}(T)^{2/3}\;.$$
Thanks to Lemma \ref{lem:E_k}, 
\begin{equation}\pr{\cup_{k=1}^{K_T} \mathcal{E}_k^C} \leq \sum_{k=1}^T  \pr{\mathcal{E}_k^C} \leq (S_{\mathbb{D}}+1) \sum_{t=1}^T T^{-4/3}/2  \leq(S_{\mathbb{D}}+1)  T^{-1/3}\;,\label{eq:reg_indes_ev}
\end{equation}
when taking $\delta = T^{-4/3}$, since $K_T\leq T$.
Hence, \begin{align*}
\E[{R_T}]&
\leq \frac{4}{3}C_4 \log (2T)  + \frac{4}{3}\left(\frac{1}{(S_{\mathbb{D}}+1)}\right)^{2/3} 
C_5 \log(2T)^{2/3}T^{2/3} + (S_{\mathbb{D}}+1) U T^{2/3} \\
&= O((\log T)^{2/3}T^{2/3})\;.
\end{align*}
\end{proof}

\subsection{Regret Analysis of FDS-Seq}\label{sec:app_reg_wT_n_u}

Instead of considering $\mathcal{E}_k$ as in the previous section, we need to consider $\mathcal{E}'_k =\{\big(f(x_k)- f(x_k+\alpha_k v_k) \leq 3 c/2(\alpha_k)^2$ and $k$ is an unsuccessful iteration$\big)$ or  $\big( k$ is a successful iteration and
$f(x_k)- f(x_k+\alpha_k v_k) \geq c/2(\alpha_k)^2\big)\}$. Instead of Lemma \ref{lem:E_k} we prove the following result.

\begin{lemma}\label{lem:E_k_tests}
The probability of $\mathcal{E}'_k$ is lower bounded by
$$\pr{ \mathcal{E}'_k|\mathcal{F}_{k-1}} \geq 1- \delta \times N_{k}^2 \times  (|\mathbb D |+1)\;,$$ where $\mathcal{F}_{k-1}$ is the $\sigma$-field representing the history.
\end{lemma}
\begin{proof} 
 Fix $v\in \mathbb{D}$.
First assume that $f(x_k)>f(x_k+ \alpha_kv) + 3/2 \rho(\alpha_k)$. In particular, $f(x_k)>f(x_k+  \alpha_kv) + \rho(\alpha_k)$. We denote $n^{\tau}_{0,k}$ and $n^{\tau}_{v,k}$ the values of $n_{0,k}$ and $n_{v,k}$ at the end of the while loop of FDS-Seq.
 Observe that 
\begin{align*}\pr{\mathcal{E}'_k|\mathcal{F}_{k-1}, n^{\tau}_{0,k} = n^{\tau}_{v,k}=N_k}    \\ \geq 1- \delta  \times  (|\mathbb D |+1)\;.
\end{align*}
Hence we only need to focus on the case when the first part of Condition \ref{eq:cond_test} is first satisfied. In this case, knowing $n^{\tau}_{0,k}, n^{\tau}_{v,k}$, the probability 
  that
$\hat{f}_{n^{\tau}_{0,k}}(x_k)<\hat{f}_{n^{\tau}_{v,k}}(x_k+ \alpha_k v) + \rho(\alpha_k)$ when the first row of Condition \ref{eq:cond_test} is first satisfied is bounded as follows.
We have 
\begin{dmath*}
\mathbb{P}
\left( \left.
f(x_k)-\hat{f}_{n^{\tau}_{0,k}}(x_k)-(f(x_k+\alpha_k v) - \hat{f}_{n^{\tau}_{v,k}}(x_k+ \alpha_k v)) \geq\\ \sqrt{2 \sigma^2 \log(1/\delta)} \sqrt{\frac{1}{n^{\tau}_{0,k}} +\frac{1}{n^{\tau}_{v,k}}}  \right| \F_k, n^{\tau}_{0,k}, n^{\tau}_{v,k} \right)
 \leq \delta.
\end{dmath*}
Since $f(x_k)>f(x_k+ \alpha_kv) + \rho(\alpha_k)$,
\begin{multline*}f(x_k)-\hat{f}_{n^{\tau}_{0,k}}(x_k)-(f(x_k+\alpha_k v) - \hat{f}_{n^{\tau}_{v,k}}(x_k+ \alpha_k v))\\
 \geq - \rho(\alpha_k) -\hat{f}_{n^{\tau}_{0,k}}(x_k)+ \hat{f}_{n^{\tau}_{v,k}}(x_k+ \alpha_k v)\;.
 \end{multline*}
So that the above deviation bound results in:  

\begin{dmath*}
\mathbb{P}
\left(\left.
\hat{f}_{n^{\tau}_{0,k}}(x_k)- \hat{f}_{n^{\tau}_{v,k}}(x_k+ \alpha_k v) - \rho(\alpha_k) \leq\\ -\sqrt{2 \sigma^2 \log(1/\delta)} \sqrt{\frac{1}{n^{\tau}_{0,k}} +\frac{1}{n^{\tau}_{v,k}}}  \right| \F_k, n^{\tau}_{0,k}, n^{\tau}_{v,k} \right)
\end{dmath*}
Finally, we apply a union bound. Since
$n^{\tau}_{0,k}$ and $n^{\tau}_{v,k}$ both belong to $[0,N_k]$ and cannot be simultaneously equal to $N_k$:
\begin{dmath*}
\mathbb{P}
\left( \left.
\hat{f}_{n^{\tau}_{0,k}}(x_k)- \hat{f}_{n^{\tau}_{v,k}}(x_k+ \alpha_k v) - \rho(\alpha_k) \leq\\ -\sqrt{2 \sigma^2 \log(1/\delta)} \sqrt{\frac{1}{n^{\tau}_{0,k}} +\frac{1}{n^{\tau}_{v,k}}}  \right| \F_k, \text{ not}(n^{\tau}_{0,k}= n^{\tau}_{v,k}= N_k) \right)\leq (N_k^2-1)\delta\;.
\end{dmath*}
This amounts to a bound of the probability of $k$ being an unsuccessful iteration and thus of ${\mathcal{E}'}_k^C$, when the first part of condition
\ref{eq:cond_test} is satisfied. Summing with the probability of ${\mathcal{E}'_k}^C$ in the other case, 
we obtain 
$\pr{\mathcal{E}'(k)}\geq N_k^2 \delta$.

The case $f(x_k)>f(x_k+ \alpha_kv) + 3/2 \rho(\alpha_k)$ can be treated  in the exact same way.

\end{proof}

To adapt the proof of Theorem \ref{th:DSwS_hor} to FDS-Seq (with a different choice of $\delta= T^{-10/3} $), it suffices to  replace \Eqref{eq:reg_indes_ev} by 
\begin{align*}\pr{\cup_{k=1}^{K_T} {\mathcal{E}'}_k^C} &\leq(S_{\mathbb{D}}+1)  \sum_{k=1}^T \pr{ {\mathcal{E}'}_k^C} \leq (S_{\mathbb{D}}+1) \sum_{k=1}^{K_T} N_{k}^2 T^{-10/3}\\
& \leq(S_{\mathbb{D}}+1) \sum_{k=1}^{T} T^2 T^{-10/3}/2  \leq(S_{\mathbb{D}}+1)  T^{-1/3}\;.
\end{align*}
The regret is hence

\begin{align*}
\E[{R_T}]&
\leq \frac{10}{3}C_4 \log (2T)  + \frac{10}{3}\left(\frac{1}{(S_{\mathbb{D}}+1)}\right)^{2/3} 
C_5 \log(2T)^{2/3}T^{2/3} + (S_{\mathbb{D}}+1) U T^{2/3} \\
&= O((\log T)^{2/3}T^{2/3})\;.
\end{align*}

\section{Noisy and Constrained Set-Up}
\label{sec:app_noi_c}
In this section we analyze the behavior of the algorithms in the presence of linear constraints.
The complexity of feasible direct search  with linear constraints in the noiseless case has  been  studied by \cite{gratton2019direct}. In this paper, instead of studying the speed at which the gradient converges to $0$ as is usual in the unconstrained case, the authors study the convergence of a lower bound of the gradient $\chi(x):=\max_{x+ v\in \D, ~\|v\|\leq 1}- \nabla f(x)^T v $ to $0$.
Indeed, the convergence of the gradient to $0$ might be unachievable when the optimum lies on the boundaries, but $\chi(x)$ is equal to $0$ if and only if $x$ is optimal.
 The  paper proves that the first iteration $K_{\epsilon}$ at which $\chi(x_k)$ is smaller than $\epsilon$ is of the order of $\epsilon^{-2}$, like in the unconstrained case. 

In the following, we denote by $U$ the global upper bound of $f(x) - f({x_\star})$ on the domain.

\subsection{Intermediate Results}
We recall that $\mathcal{E}_k$ denotes the event $$\mathcal{E}_k =\{ |f(x_k+ \alpha_k v)- \hat{f}(x_k+ \alpha_k v)|\leq c/4(\alpha_k)^2\},~\forall v \in {\mathbb D } \cup\{0\}\}.$$
Lemmas \ref{lem:E_k},  \ref{lem:concentration}, \ref{lem:desc} are left unchanged by the transition to constrained domains. 
A version of Lemma 3.4. of \citep{gratton2019direct} reads : 
\begin{lemma} \label{lem:unsuccessful iteration_c_n}On 
$\cap_{k\leq K}\mathcal{E}_k$, and if $f$ satisfies Assumption 
\ref{ass:reg}, then the following holds: 
if the $k$-th iteration is unsuccessful, then $$\chi(x_k)\leq \left(\frac{\beta}{2 \kappa}+\frac{Bg}{\eta_{\min}}\right) \alpha_k+\frac{3\rho(\alpha_k)}{2\kappa \alpha_k} := L_2 \alpha_k,$$ where $\eta_{\min} := \lambda(N)$ where $N$ is the set of  all possible approximate normal cones $N(x,\alpha), ~\forall  \in \D, ~\alpha \in \mathbb{R}^d$.
\end{lemma}
\begin{proof}
It is straightforward to prove 
$$\|P_{T(x_k, \alpha_k)}(-\nabla f(x_k)) \|\leq \frac{1}{\kappa}\left(\frac{\beta}{2} \alpha_k + \frac{3\rho(\alpha_k)}{2\alpha_k}\right)= \frac{1}{2\kappa}\left(\beta \alpha_k + 3c \alpha_k^2\right),$$
with the same elements as in the proof Lemma \ref{lem:bound_increase_grad_noise}, by noticing that $\mathbb{D}_k$ contains $\mathcal{G}_k$, that generates $T(x_k, \alpha_k)$.
To prove the bound on $\chi(x_k)$,  we use 
the  Moreau decomposition,stating  that any vector $v\in \mathbb{R}^d$ can be decomposed as $v=P_{T_k}[v] +P_{N_k}[v]$ with $N_k=N(x_k,\alpha_k)$ and $P_{T_k}[v]>P_{N_k}[v] =
0$, and write 
\begin{align}
\chi(x_k) &=
\max_{x+ v\in \D, ~\|v\|\leq 1} ( P_{T_k}[-\nabla{f}(x_k)] +(P_{T_k}[v] v^T+
P_{N_k}[v])^T P_{N_k}[-\nabla{f}(x_k)])\nonumber\\
&\leq \max_{x+ v\in \D, ~\|v\|\leq 1}(v^T P_{T_k}[-\nabla{f}(x_k)] +P_{N_k}[v]^T P_{N_k}[-\nabla{f}(x_k)])\nonumber\\
&\leq \max_{x+ v \in \D, ~\|v\|\leq 1}\|P_{T_k}[-\nabla{f}(x_k)]\| 
+\|P_{N_k}[v]\|
\|P_{N_k}[-\nabla{f}(x_k)]\|
\label{eq:chi}.
\end{align}
 The first  term of the right hand side of \Eqref{eq:chi} is bounded in the following way
$$ \|P_{T(x_k, \alpha_k)}(-\nabla f(x_k)) \|\leq  \frac{1}{2\kappa}\left(\beta \alpha_k + 3c \alpha_k^2\right)$$ 
consequently.
 \begin{lemma}[Proposition  B.1 of \citep{lewis2000pattern}]  Let $x\in \mathcal{D}$ and $\alpha >0$.   Then,  for  any  vector $v$ such  that $x+v \in \D$, one has $$\|P_{N(x,\alpha)}[v]\| \leq \frac{\alpha}{\eta_{\min}}.$$
 \end{lemma}
 This in turn provides a bound of the second term of the right hand side of \Eqref{eq:chi}:
 $$\|P_{N_k}[v]\| \|P_{N_k}[-\nabla{f}(x_k)]\| \leq \frac{\alpha}{\eta_{\min}} B_g,$$
which suffices to conclude the proof.
\end{proof}


\begin{lemma}\label{lem:vers_th41_grat}
Assume that 
$\cap_{k\leq K}\mathcal{E}_k$ holds, and  $f$ satisfies Assumption 
\ref{ass:reg}. Set $\epsilon >0$.
Let $h$ denote the mapping from $\epsilon$ to $$ h(\epsilon)=\left( \frac{2 L_2^2 U}{c\theta^1} \right) \epsilon ^{-2} + \frac{\log(\frac{\alpha_0 L_2}{\theta})}{\log(1/\theta)} +  \frac{2U}{\alpha_0^2} = E_1 \epsilon^{-2} + E_2,$$
where $E_1 = \left( \frac{2 L_2^2 U}{c\theta^1} \right) $ and $E_2 =  \frac{\log(\frac{\alpha_0 L_2}{\theta})}{\log(1/\theta)} +  \frac{2U}{\alpha_0^2}$.
Denote by $k(\epsilon)$ the first iteration of the algorithm where $\chi(x_k)\leq \epsilon$.
If $h(\epsilon) \leq K $, then 
$$k(\epsilon) \leq h(\epsilon).$$
\end{lemma}
The proof is a mere adaptation of  the proof of Theorem 1 of \citep{gratton2015direct},  with different constants (we use Lemma \ref{lem:unsuccessful iteration_c_n} and \ref{lem:sumsteps_noise}).
\begin{lemma}\label{lem:use} If $f$ satisfies Assumption \ref{ass:reg}, $$a \|x_k - {x_\star}\| \leq \chi(x_k).$$
\end{lemma}
\begin{proof}
$$a \|x_k - 
{x_\star}\| \leq\frac {f(x)- f({x_\star})}{\|x_k - 
{x_\star}\|}\leq \frac {-\nabla f(x)({x_\star} -x_k)}{\|x_k - 
{x_\star}\|}\leq \chi(x_k),$$
by definition of $\chi(x_k).$
\end{proof}
\begin{lemma}\label{lem:bound_increase_grad_cons} If $f$ satisfies Assumption \ref{ass:reg} and if ${x_\star}$ is in the interior of $\D$ and the algorithm achieves descent at each iteration, then 
$$\forall k'>k, \|\chi(x_{k'}) \| \leq \left(\frac{\beta}{a}\right)^{3/2}\|\chi(x_{k}) \|.$$
\end{lemma}
\begin{proof}
Like in the proof of Lemma \ref{lem:bound_increase_grad}, we get 
\begin{equation*}
 a\| x_k- x^*\|\leq \|\nabla f(x_k)\|.\end{equation*}
 Hence, 
\begin{align*}
 \beta \|x_{k'} - {x_\star}\|
& \leq \frac{\beta}{\sqrt{a}} \sqrt{f(x_{k'} )- f({x_\star})}
\leq \frac{\beta}{\sqrt{a}} \sqrt{f(x_{k} )- f({x_\star})} 
\\
& \leq \frac{\beta}{\sqrt{a}} \sqrt{\beta (x_k-{x_\star})^2}\leq \frac{\beta^{3/2}}{\sqrt{a}} \sqrt{\frac{\chi(x_k)^2}{a}}\\
& \leq \left(\frac{\beta}{a}\right)^{3/2} \chi(x_k),
\end{align*}
where the first inequality comes from the strong convexity, and the second one comes from the fact that the algorithm is a descent, the third one comes from Lemma \ref{lem:use}.

Now let $ v= \argmax_{x+ v \in \D, ~ \|v\| \leq 1} - v^T\nabla f(x_{k'})$.
$$- v^\top \nabla f(x_{k'})= - v^\top \nabla f(x_{k'}) +  v^\top  \nabla f(x_{*}) \leq \| \nabla f(x_{k'}) - \nabla f(x_{*}) \| \leq \beta \|x_k - {x_\star}\|,$$
because $\nabla f(x_{*})=0.$
This concludes the proof.
\end{proof}
\subsection{Regret Analysis when the Optimum is in the interior of $\D$}\label{app:interior_proof}

Let us assume that ${x_\star}$ is in the interior of $\D$.
Let us denote by $\Delta$ the distance from ${x_\star}$ to the closest boundary, and by $r= \Delta/4$.
\begin{lemma}\label{lem:inte}
If $\chi(x) \leq a r$ then $\|x- {x_\star}\|\leq r$.
\end{lemma}
If $\|x- {x_\star}\|\geq r$
then $r\leq\frac{1}{a} \chi(x)$ thanks to Lemma \ref{lem:use}.
\begin{lemma}\label{lem:inte_2}
For any $k\geq k\left(\left(a/\beta\right)^{3/2}a r\right)$, $\|x_k - {x_\star}\|\leq r$.
\end{lemma}
Thanks to Lemma \ref{lem:bound_increase_grad}, after  $k\left(\left(a/\beta\right)^{3/2}a r\right)$ iterations, $\chi(x_k)\leq ar$.
And thanks to the previous lemma, we thus have $\|x_k- {x_\star}\|\leq r$.

\begin{lemma}
Set $k_s$ the index of the  first successful iteration following \\ $ k\left(\left(a/\beta\right)^{3/2}a r\right)$ where $\alpha_k\leq \Delta/2$. After iteration $k_s$,  $T(x_k, \alpha_k)$ spans all directions in $\R^d$, so that the instantaneous regret is the same as that of the algorithm in the unconstrained case with initial point $x_{k_s}$ and initial step-size $\alpha_{k_s}$. The iteration of this first successful iteration comes before $k_i := k\left(\left(a/\beta\right)^{3/2} ar\right) + \frac{\log( \alpha_0/\Delta)}{\log{1/\theta}}$.
\end{lemma}

\begin{proof}
If $k$ is a successful iteration $\alpha_k \leq 2r = \Delta / (2)$, since $\|x_k - {x_\star}\|\leq r$ and $\|x_{k+1} - {x_\star}\|\leq r$.
And if $k$ is an unsuccessful iteration, it comes after one of those successes and a sequence of unsuccessful iterations, which yields $\alpha_k \leq  \Delta / (2)$.
\end{proof}
\begin{lemma}\label{lem:reg_con_noise_k}
On $\cap_{k\leq K}\mathcal{E}_k$,

$$\tilde{R}_K \leq C_7 \log (2/\delta)(1/\theta)^{-4 C_f} + C_5\log(2/\delta)^{1/3} \left(\sum_{k=1}^K N_k\right)^{2/3},$$
where $C_7 = (S_{\mathbb{D}}+1) U \frac{32 }{c^2\alpha_0 ^4 }\frac{1}{(1/\theta) -1}$\\
 and $$C_f = E_1 \left( \left(\frac{a}{\beta}\right)^{3/2} ar \right)^{-2} + E_2 + \frac{\log( \alpha_0/\Delta)}{\log{1/\theta}} + \frac{\beta \Delta^2}{\alpha_0 }.$$
\end{lemma}

\begin{proof}
In the proof of Lemma \ref{lem:DSwS_hor}, we isolated the steps preceding the first unsuccessful iteration. Similarly here, we treat the iterations before 
the first unsuccessful iteration after $k_i$, denoted by $k_f'$, separately from  other iterations.
\begin{align*}
\tilde{R}_K &\leq \sum_{k=0}^{k_f'}N_k (f(x_k) - f({x_\star})) + \sum_{k=0}^{k_f'}\left(N_k\sum_{v\in \mathbb{D}_k}  f(x_k+ \alpha_k v) - f({x_\star})\right)\\
& + \sum_{k=k_f'}^K \left(N_k(|\mathbb{D}_k|+1)(f(x_k) - f({x_\star}))\right) +\sum_{k=k_f'}^K N_k\left(\sum_{v\in \mathbb{D}_k}  f(x_k+ \alpha_k v) - f(x_k)\right).
\end{align*}
Because $f(x) - f({x_\star})$ is bounded by $U$, 
\begin{align*}
&\sum_{k=0}^{k_f'}N_k (f(x_k) - f({x_\star})) + \sum_{k=0}^{k_f'}\left(N_k\sum_{v\in \mathbb{D}_k}  f(x_k+ \alpha_k v) - f({x_\star})\right)\\
&\leq \sum_{k=0}^{k_f'} (|\mathbb{D}_k|+1)N_k U  .
\end{align*}
By rewriting $N_k$, 
\begin{align*}
\sum_{k=0}^{k_f'} (|\mathbb{D}_k|+1)N_k U &\leq \sum_{k=0}^{k_f'} (|\mathbb{D}_k|+1) U \frac{32 \log (1/\delta)}{c^2\alpha_k^{4}}\\
&\leq  \sum_{k=0}^{k_f'} (|\mathbb{D}_k|+1) U \frac{32 \log (1/\delta)}{c^2\alpha_0 ^4 \theta^{4k}}
\\
&\leq  (S_{\mathbb{D}}+1) U \frac{32 \log (1/\delta)}{c^2\alpha_0 ^4 }\frac{(1/\theta)^{-4 k_f'}}{(1/\theta) -1}.
\end{align*}
Also on $\cap_{k\leq K}\mathcal{E}_k$,
$$k_f'\leq k_i + \frac{f(x_{k_i}) - f({x_\star})}{\alpha_0} \leq k_i + \frac{\beta \Delta^2}{\alpha_0} \leq k\left(\left(\frac{a}{\beta}\right)^{3/2} ar\right) + \frac{\log( \alpha_0/\Delta)}{\log{1/\theta}} + \frac{\beta \Delta^2}{\alpha_k },$$
where the first inequality comes from the same argument used to prove Lemma \ref{lem:k_f}, the second inequality comes from the smoothness of $f$ and the third one comes from the definition of $k_i$.

On $\cap_{k\leq K}\mathcal{E}_k$, 
$$k\left(\left(\frac{a}{\beta}\right)^{3/2} ar\right)\leq E_1\left(\left(\frac{a}{\beta}\right)^{3/2} ar\right)^{-2} + E_2,$$
with $E_1$, $E_2$ defined in \citep{gratton2019direct}.
Finally, we focus on the part of the regret accumulated before $k_f'$.
 On $\cap_{k\leq K}\mathcal{E}_k$
\begin{multline*}
 \sum_{k=k_f'}^K \left(N_k(|\mathbb{D}_k|+1)(f(x_k) - f({x_\star}))\right) + \sum_{k=k_f'}^K N_k\left(\sum_{v\in \mathbb{D}_k}  f(x_k+ \alpha_k v) - f(x_k)\right)\\
 \leq 
 C_5\log(2/\delta)^{1/3}\left({\sum_{k=1}^K N_k}\right)^{2/3} ,
\end{multline*}
by following exactly the same steps  as those needed to bound the regret in the unconstrained case.
\end{proof}

\begin{repeatthm}{th:FDS-Plan_hor}
 Under Assumptions \ref{ass:reg}, \ref{ass:bounded_D} 
 and \ref{ass:gk}, and if ${x_\star} \in \text{int}(\D)$,   the cumulative regret $R_T$ of FDS-Plan (respectively FDS-Seq) after the first $T$  evaluations of $f$,  satisfies 
$$\E[{R_T}]= O\big(\log(T)^{2/3}T^{2/3}\big)$$
for the choice  $\delta = T^{-4/3}$ (respectively $\delta = T^{-10/3}$ for FDS-Seq).
\end{repeatthm}

\begin{proof} \textbf{for FDS-Plan. }
We denote by $K_T$ the last round reached by the algorithm with $T$ evaluations..
Lemma \ref{lem:reg_con_noise_k} proves that on the event $\cap_{k\leq K_T} \mathcal{E}_k$, 
$$\tilde{R}_K \leq C_7 \log (2/\delta)(1/\theta)^{-4 C_f} + C_5\log(2/\delta)^{1/3} \left(\sum_{k=1}^K N_k\right)^{2/3}.$$

Thanks to Lemma \ref{lem:E_k}, 
$$\pr{\cup_{k=1}^{K_T} \mathcal{E}_k^C} \leq(S_{\mathbb{D}}+1) \pr{\cup_{k=1}^T \mathcal{E}_k^C} \leq (|\mathbb{D}_k|+1) \sum_{t=1}^T T^{-4/3}/2  \leq(S_{\mathbb{D}}+1)  T^{-1/3}$$
since $K_T\leq T$.
Hence, \begin{align*}
\E[{R_T}]&
\leq \frac{4}{3}C_7 (1/\theta)^{-4 C_f} \log (2T)  + \frac{4}{3}\left(\frac{1}{(|\mathbb{D}_k|+1)}\right)^{2/3} 
C_5 \log(2T)^{2/3}T^{2/3} \\
&+ ((S_{\mathbb{D}}+1)) U T^{2/3} \\
&= O((\log T)^{2/3}T^{2/3}).
\end{align*}
\end{proof}
\paragraph{Adaptation of the proof for FDS-Seq} The way of adapting the  proof of FDS-Plan to the case of FDS-Seq of Section \ref{sec:app_reg_wT_n_u} applies verbatim.

{
\section{Details on the implementation of HOO in Section \ref{sec:illustration}}\label{app:HOO}

To implement HOO in the simulations of Section \ref{sec:illustration}, the tree of partitions that we used is
built in the following way. We set the parameter $\rho$ of HOO as
suggested by \cite{bubeck2011x} to $2^{-2/d}$.  A binary tree of depth
$H=\frac{\log(1/T)}{2\log(\rho)}$ of partitions of $[0,1]^d$ is
obtained by recursively halving the cells at each depth $h$ of the
tree along dimension $h \pmod 2$.  At depth $h$, this
approach yields a partition formed by rectangular cells represented by
their lower left corner $[a_{i},b_{i}]$. Then, in order to remove
unwanted cells, we traverse the tree, starting from the leaves, and
remove every cell having an empty intersection with the domain.  Due
to the geometry of the simplex, knowing if a cell intersects the
domain boils down to checking if its representation $[a_{i},b_{i}]$
belongs to it.  When the algorithm selects cell $(h,i)$ at time $t$,
the representation of that cell is chosen as a sampling point.  In the
simulation, the smoothness parameter $\nu_1$ of HOO is set to $16$.\looseness-1}

{
\section{Additional Experiments}\label{app:suppl_exp}

Here, as in Section \ref{sec:illustration}, we focus on the case in which there are three resources ($d=2$).
The loss functions for resources $1$ and $3$ are of the same form as in Section \ref{sec:illustration} and 
$w_i(x)= -\tau_i \frac{\log(1+ \gamma x)} {\log(1+ \gamma )}$ with
$\gamma=2$, $\tau_1= 1$, $\tau_3 = 0.3$, but now the second resource is associated to $w_2(x) = 0.1 \, x$.
This choice of reward functions results in an optimal choice whose second component is zero.
We set the horizon to $T=100,000$ and use a Gaussian noise with
standard deviation $\sigma= 0.1$.

We show the trajectories of FDS-Plan and FDS-Seq in Figure \ref{fig:trajectories_2}. Notice that the trajectories do not change drastically compared to those of Section \ref{sec:illustration}, which seems to indicate that the location of the optimal allocation on the border of the feasible set is not a problem in practice.
We complement these plots with regret plots (Figure \ref{fig:reg_plots_d3}) of all the algorithms detailed in Section \ref{sec:illustration} run first on the environment described in this same section and second on the environment described above with the optimum on the border of the simplex.
Once again, this seems to show that the optimum lying in the border is not an issue in practice.
Incidentally, this last plot also shows that, as expected, UCB should be the preferred algorithm in dimension $d=2$, as it is simpler and gives excellent results.

 \begin{figure}[ht]
\centering
\subfigure[FDS-Plan]{
\includegraphics[width=0.23\linewidth]{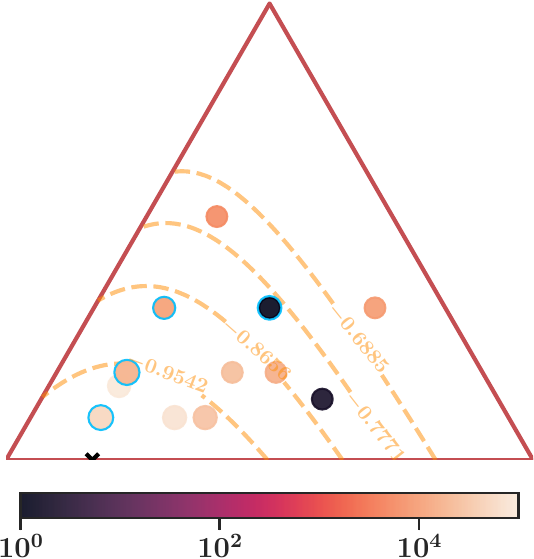}
\label{fig:FDS-Plan_4}
}\hspace{0.12\linewidth}
\subfigure[FDS-Seq]{%
\includegraphics[width=0.23\linewidth]{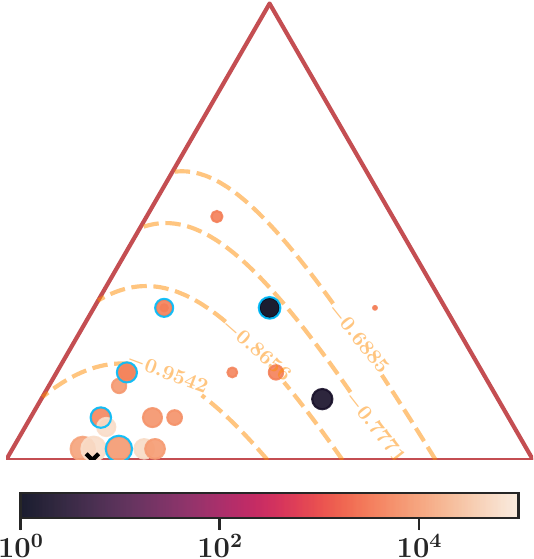}
\label{fig:FDS-Seq_4}
} 
\caption{Single trajectories with the optimum on the border}
\label{fig:trajectories_2}
\end{figure}

\begin{figure}[ht]
\centering
\subfigure[Optimum in the interior]{
\includegraphics[width=0.45\linewidth]{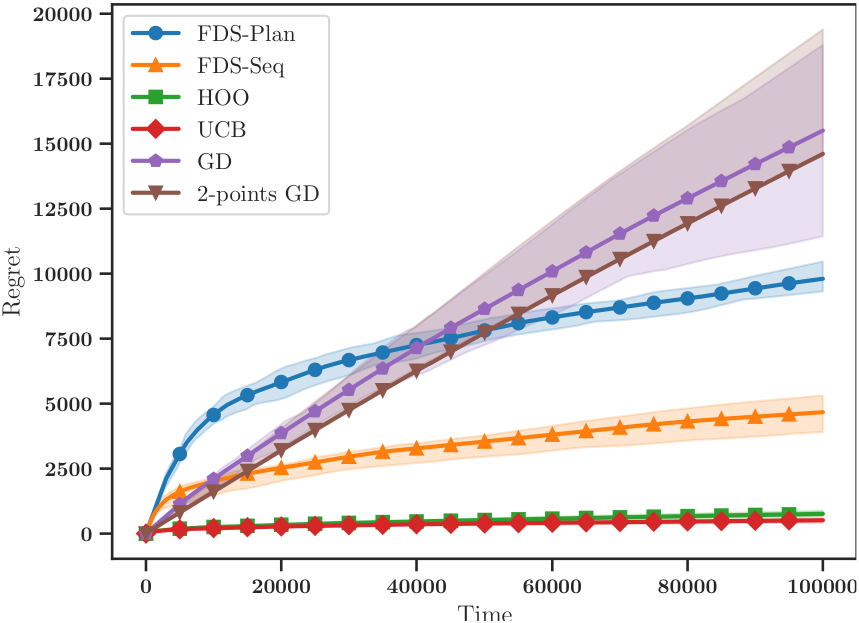}
\label{fig:FDS-Plan_3}
}\quad
\subfigure[Optimum on the border]{%
\includegraphics[width=0.45\linewidth]{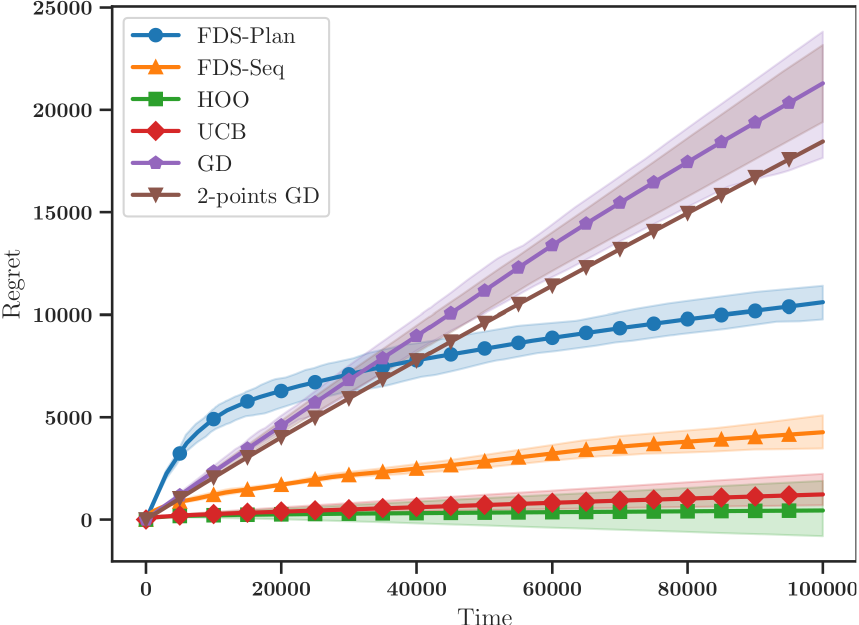}
\label{fig:FDS-Seq_3}
} 
\caption{Regret plots in dimension $d=2$}
\label{fig:reg_plots_d3}
\end{figure}
}